\documentclass[11pt,a4paper]{article}

\title{A Single Stepsize Suffices for Unprojected Linear TD(0):\\Simultaneous Robust and Fast Rates via Polyak--Ruppert Averaging}

\author{Wei-Cheng Lee\\
\small King Abdullah University of Science and Technology (KAUST)\\
\small Thuwal, 23955-6900, Kingdom of Saudi Arabia\\
\small \texttt{weicheng.frank.lee@gmail.com}
\and
Francesco Orabona\\
\small King Abdullah University of Science and Technology (KAUST)\\
\small Thuwal, 23955-6900, Kingdom of Saudi Arabia\\
\small \texttt{francesco@orabona.com}}
\date{}

\usepackage{times, mathtools, dsfont}
\usepackage{algorithm, algorithmic}
\usepackage{titlesec}
\usepackage{tikz}

\usetikzlibrary{arrows.meta, positioning}
\usepackage{booktabs}
\usepackage{xcolor}
\usepackage{colortbl}
\usepackage{makecell}
\usepackage{subcaption}
\usepackage{etoolbox}
\usepackage{amsfonts, amsmath, amssymb}
\usepackage{amsthm}
\usepackage{bm}
\usepackage{natbib}
\usepackage{url}
\usepackage{a4wide}
\usepackage{hyperref}
\AtBeginEnvironment{assumption}{\phantomsection}

\renewcommand{\Pr}{\field{P}}

\newcommand{\bb}{\boldsymbol{b}}

\newcommand{\be}{\boldsymbol{e}}
\newcommand{\bg}{\boldsymbol{g}}

\newcommand{\bx}{\boldsymbol{x}}

\newcommand{\bA}{\boldsymbol{A}}

\newcommand{\bD}{\boldsymbol{D}}

\newcommand{\bI}{\mathbf{I}}
\newcommand{\bL}{\boldsymbol{L}}

\newcommand{\bP}{\boldsymbol{P}}

\newcommand{\bxi}{\boldsymbol{\xi}}

\newcommand{\bv}{\boldsymbol{v}}
\newcommand{\bzero}{\boldsymbol{0}}

\newcommand{\bdelta}{\boldsymbol{\delta}}

\newcommand{\bzeta}{\boldsymbol{\zeta}}
\newcommand{\bSigma}{\boldsymbol{\Sigma}}
\newcommand{\btheta}{\boldsymbol{\theta}}

\newcommand{\bphi}{\boldsymbol{\phi}}
\newcommand{\bPhi}{\boldsymbol{\Phi}}

\newcommand{\bV}{\boldsymbol{V}}

\newcommand{\field}[1]{\mathbb{#1}}

\newcommand{\R}{\field{R}}

\newcommand{\E}{\field{E}}

\newcommand{\mixing}{\tau_{\mathrm{mix}}}
\newcommand{\F}{\mathcal{F}}

\newcommand{\norm}[1]{\left\|{#1}\right\|}

\newcommand{\indicator}{\mathbf{1}}

\newcommand{\Ip}[2]{\langle {#1}, {#2} \rangle}

\NewDocumentCommand{\Ec}{o m o}{
  \IfNoValueTF{#3}{\mathbb{E}}{\mathbb{E}_{#3}}
  \IfNoValueTF{#1}
    {\!\left[#2\right]}
    {\!\left[#2\;\middle|\;#1\right]}
}

\usepackage{pifont}

\newcommand{\rbase}{\ensuremath{R_{\mathrm{base}}}}
\newcommand{\rbases}{\ensuremath{R_{\mathrm{base}}^2}}

\newcommand{\rmax}{\ensuremath{R_{\mathrm{max}}}}
\newcommand{\rmaxs}{\ensuremath{R_{\mathrm{max}}^2}}
\newcommand{\rinf}{\ensuremath{r_{\infty}}}
\newcommand{\phiinf}{\ensuremath{\phi_{\infty}}}

\newcommand{\phiinfs}{\ensuremath{\phi_\infty^2}}
\newcommand{\exit}{T_{\mathrm{exit}}}

\theoremstyle{plain}
\newtheorem{theorem}{Theorem}[section]
\newtheorem{lemma}[theorem]{Lemma}
\newtheorem{corollary}[theorem]{Corollary}
\theoremstyle{definition}
\newtheorem{assumption}[theorem]{Assumption}
\theoremstyle{remark}
\newtheorem{remark}[theorem]{Remark}

\begin{document}

\maketitle

\begin{abstract}
We study linear TD(0) under Markovian sampling, where data are generated along a single trajectory.
We provide high-probability guarantees for a plain \emph{unprojected} TD(0) algorithm with Polyak--Ruppert (PR) averaging, using a \emph{single} stepsize schedule $\eta_t \propto 1/(\mixing\log(t)\,\sqrt{t})$ that depends on mixing time but requires \emph{no prior knowledge of the curvature parameter $\omega$}.
Our first result shows that such a choice of the stepsize guarantees that the TD(0) iterates are automatically and uniformly bounded \emph{with high probability}, without projections and without any stability argument based on $\omega$.
Building on this result, we establish a simultaneous high-probability convergence guarantee for the PR average: the same stepsize yields both a robust curvature-free $\widetilde{\mathcal O}(\mixing/\sqrt{T})$ rate and a fast curvature-dependent $\widetilde{\mathcal O}(\mixing^2/(\omega T))$ rate, with the bound taking the minimum of the two.
The core technical ingredient is a Poisson-equation toolkit for geometrically mixing Markov chains, which decomposes Markov noise into a martingale term plus a controlled remainder and enables a new self-bounding inductive argument for pathwise stability.

\end{abstract}

\medskip

\section{Introduction}

Temporal-difference (TD) learning is one of the central algorithmic primitives in reinforcement learning, used for policy evaluation and as a building block for control methods such as actor--critic. When the state space is large, TD is typically combined with function approximation; in the linear case, the resulting TD(0) recursion admits a clean stochastic-approximation interpretation and converges
asymptotically under standard conditions~\citep{sutton1988learning,tsitsiklis1996analysis,kushner2010stochastic}.

Despite this classical theory, obtaining \emph{finite-time}, \emph{high-probability} guarantees for TD(0) remains challenging, especially in the practically relevant \emph{Markovian sampling} regime where data are generated along a single trajectory.

A major obstacle is that TD(0) is not a standard stochastic gradient method:
Even in the linear setting, the update is generally biased and non-symmetric, and the noise is temporally correlated. Moreover, the magnitude of the TD(0) updates depends on the magnitude of the iterates.
Therefore, existing non-asymptotic analyses rely on additional structure to control the iterates. One common approach is to enforce boundedness via \emph{projections} onto a known ball~\citep[e.g.,][]{bhandari2018finite}, which simplifies concentration arguments but changes the algorithm and requires \emph{a priori} knowledge
of a suitable radius (often tied implicitly to the problem curvature).
A second approach is to exploit a \emph{contractive/curvature} structure of the mean dynamics to argue stability without projections; however, this typically yields stepsize conditions or stability bounds that depend on an unknown curvature parameter, which becomes arbitrarily slow when the curvature is arbitrarily small.

Ideally, one would like to use a \emph{single} stepsize rule that \emph{simultaneously} allows us to obtain rates that are
(i) \emph{robust} (curvature-free) and (ii) \emph{fast} when curvature is favorable. Such a stepsize would adapt to the better of the robust and fast high-probability upper bounds, without prior knowledge of the curvature.

\paragraph{Robust and fast rates.}
To explain the distinction between robust and fast rates, we first briefly describe the potential function we study~\citep{ollivier2018approximate,liu2021temporal} (formally introduced later in~\eqref{eq:dirichlet}),
which naturally captures the TD fixed point error under Markovian sampling.
This potential always has positive curvature $\omega>0$, so
one can hope for a ``fast'' rate of order\footnote{The notation $\widetilde{\mathcal{O}}$ suppresses logarithmic factors..} $\widetilde{\mathcal{O}} (\frac{1}{\omega T})$ (up to mixing and logarithmic factors).
On the other hand, even without curvature assumptions, one can always achieve a ``robust'' rate of order $\widetilde{\mathcal{O}}(\frac{1}{\sqrt{T}})$, again up to mixing and logs. It is obvious that, in the finite-time regime, one rate can be better than the other depending on $\omega$, so both rates should be pursued in different situations.

Now, prior work typically achieves these two regimes with \emph{different} stepsize choices and additional algorithmic modifications~\citep[e.g.,][]{bhandari2018finite}. Moreover, the projection-free high-probability analyses under Markovian noise that we are aware of rely on stability arguments whose constants deteriorate with $\omega$~\citep[e.g.,][]{samsonov2024improved,Durmus2025}.

As far as we know, \emph{obtaining both rates simultaneously with high probability with a single stepsize choice that is independent of $\omega$ and without projections has not previously been established for TD(0)}.
In particular, the key technical difficulty in achieving such a result is to control the random iterates, $\btheta_t$, \emph{pathwise}. Without projections, the update magnitude scales with $\|\btheta_t\|$, and a naive concentration analysis becomes circular: to bound the error one needs bounded updates, but bounded updates require bounded iterates. Moreover, the need to obtain robust rates (i.e., independent of $\omega$) rules out any strategy that shows that the iterates are bounded through a contractive argument based on $\omega$.

\paragraph{Our approach: a self-bounding argument via Poisson equations.}

In this paper, we show that a simple stepsize, proportional to $\frac{1}{\mixing\log t\sqrt{t}}$ and independent of $\omega$, guarantees that TD(0) achieves both fast and robust rates with high probability.

Our main technical novelty to circumvent the above issues is a new \emph{self-bounding} argument that closes this loop \emph{with high probability} \emph{without} appealing to curvature $\omega$ and \emph{without} artificial projections.
Concretely, we develop a Poisson-equation toolkit for geometrically mixing Markov chains that decomposes additive Markov noise into a martingale term plus a controlled remainder.
This decomposition allows us to control the Markovian bias terms that appear in the basic potential expansion of TD, and to run an induction in which boundedness at times $\le t-1$ implies boundedness at time $t$.
The resulting uniform bound on $\sup_t \ \|\btheta_t\|$ holds with high probability and is independent of the curvature $\omega$.

Once boundedness is established, we can analyze the convergence guarantee.
The bounded-iterates event provides the missing ingredient needed to control the quadratic variation of the martingale terms (again through the Poisson decomposition), yielding high-probability bounds for the averaged iterate $\bar{\btheta}_T$.
Importantly, the \emph{same} stepsize delivers:
\vspace{-0.1cm}
\begin{itemize}
\item a curvature-free robust rate of $\widetilde{\mathcal{O}}(\frac{1}{\sqrt{T}})$, through a classic averaged-stochastic-gradient-descent-like analysis, as one would expect from the choice of the stepsize;
\vspace{-0.2cm}
\item a curvature-dependent fast rate $\widetilde{\mathcal O}(1/(\omega T))$, without requiring $\omega$ to set the stepsize, thanks to a Polyak--Ruppert averaging analysis.
\end{itemize}
\vspace{-0.1cm}

\section{Related Work}
\label{sec:rel}

In this section, we briefly survey the main results on the convergence guarantees for TD learning with linear function approximation, as well as the main tools we use in our proofs. Later, in Section~\ref{sec:comparison}, we will give a detailed comparison in terms of convergence rates.

Early convergence results for TD learning with linear function approximation trace back to \citet{tsitsiklis1996analysis}, who interpreted TD updates through the lens of stochastic approximation~\citep{kushner2010stochastic}. While foundational, that theory is largely asymptotic and does not yield explicit non-asymptotic rates. Finite-time guarantees were later developed in a sequence of works~\citep{korda2015td,lakshminarayanan2018linear,dalal2018finite}, but these analyses typically assume i.i.d.\ samples drawn from the stationary distribution. This assumption sidesteps the temporal dependence present in most reinforcement-learning pipelines, where data are generated sequentially along a single Markov-chain trajectory. Accounting for such Markovian correlations substantially complicates the analysis, even for TD(0).

The seminal work of \citet{bhandari2018finite} provided the first finite-time treatment under Markovian sampling. They obtained both fast and robust rates, using two different stepsizes. However, their arguments (and even subsequent refinements for the robust rates such as \citet{liu2021temporal}) rely on an explicit projection step to keep the iterates controlled.

The need for a projection is removed in subsequent work in the fast regime by exploiting the contractive nature of the update due to the presence of the curvature.
In particular, using a control-theoretic framework, \citet{srikant2019finite} derived finite-time error bounds for linear TD with Markovian data without projections by establishing a contraction-like behavior via Lyapunov theory. However, their stepsize depends on the curvature of the potential function, which is generally unknown in practice. Building on this direction, \citet{patil2023finite} eliminated the need to know the curvature parameter when choosing stepsizes, at the cost of introducing a data-dropping modification of TD. Related progress has also focused on simplifying proofs and relaxing algorithmic modifications: \citet{mitra2025simple} proposed a streamlined inductive two-step analysis, \citet{li2025towards} used exponentially decaying stepsizes to avoid data dropping, and \citet{sun2022finite} extended fast-rate analyses to neural networks in the NTK regime. More recently, \citet{samsonov2024improved} strengthened the analysis of \citet{patil2023finite} and obtained high-probability bounds without using projections. Closest to the high-probability projection-free line, \citet{ChandakTD0Concentration2025} obtained all-time high-probability control for unprojected TD($0$) without data dropping. Their analysis, however, remains contractive in nature, and the resulting constants and burn-in depend on the curvature.

The only result that removes the need for projections with stepsizes independent of the curvature of the potential function, and without using curvature in the analysis, is \citet{lee2025finite}, which proves that the iterates of TD(0) are bounded in expectation. Our approach is inspired by their method, but differs substantially because we need high-probability bounds.

Several closely related works study cheap linear-update methods that are informative but not direct head-to-head baselines for plain TD($0$). \citet{raj2022faster} analyze linear composition optimization and gradient TD methods for off-policy MSPBE minimization, using primal--dual updates rather than the semi-gradient TD($0$) recursion studied here, and obtain adaptive finite-time guarantees but require bounded iterates or prior knowledge of the curvature. In the tabular setting, \citet{li2020sample} show that asynchronous Q-learning, and hence tabular TD, can have a leading statistical term matching the i.i.d.\ sampling complexity, with mixing entering only as an additive transient cost; they also provide data-driven stepsizes that avoid prior knowledge of $\mixing$.  These results address the same broad practical question of obtaining fast, stable, cheap linear updates under limited tuning information, but they do not directly apply to unprojected linear TD(0).

As far as we know, there is no prior work that analyzes the possibility of getting both fast and robust rates with a stepsize schedule independent of the curvature and without using projections for linear TD learning.

The use of Polyak--Ruppert averaging in the analysis of TD learning can be traced back at least to \citet{konda2002actorcritic}. To the best of our knowledge, \citet{korda2015td} were the first to leverage this technique to establish finite-time guarantees. In particular, Polyak--Ruppert (tail) averaging has been employed to enable stepsize choices that avoid explicit dependence on the curvature parameter $\omega$~\citep[see, e.g.,][]{lakshminarayanan2018linear,patil2023finite,samsonov2024improved}.

Finally, we make use of the Poisson equation to obtain high-probability bounds. This is a standard tool, used widely in the literature on TD learning and stochastic approximation with Markovian noise~\citep[see, e.g.,][]{ChandakTD0Concentration2025, blaser2024asymptotic}.

\section{Setting and Assumptions}
\label{sec:assumptions}

We work directly with the finite discounted Markov reward process (MRP) induced by a fixed policy $\mu$. Thus the action variables have already been averaged under $\mu$. Let $\mathcal{S}\coloneqq\{1,\dots,n\}$ be a finite state space, let $P^\mu:\mathcal{S}\times\mathcal{S}\to[0,1]$ be the induced transition kernel, and let $r:\mathcal{S}\times\mathcal{S}\to\mathbb{R}$ be the one-step reward function. Given an initial state $s_0$, the state process satisfies $\Pr(s_{t+1}=j\mid s_t=i)=P^\mu(i,j)$ for $i,j\in\mathcal{S}$. We write $\bP^\mu\in\mathbb{R}^{n\times n}$ for the corresponding transition matrix, with entries $\bP^\mu(i,j)=P^\mu(i,j)$, and fix a discount factor $0<\gamma<1$.

We are interested in the \emph{policy evaluation} problem in reinforcement learning \citep{sutton1998reinforcement,MannorMT-RLbook}. For the MRP induced by $\mu$, the value function is
\[
V^\mu(s) \coloneqq \Ec[s_0=s]{\sum_{t=0}^\infty\gamma^t r(s_t,s_{t+1})}.
\]
The value function $\bV^\mu$ can be viewed as a vector in $\mathbb{R}^n$ and is the unique fixed point of the Bellman expectation operator $T^\mu:\mathbb{R}^n\to\mathbb{R}^n$ defined by
\[
(T^\mu \bV)(s) \coloneqq r^\mu(s)+\gamma\sum_{s'\in\mathcal{S}}P^\mu(s,s')\,\bV(s'),\qquad s\in\mathcal{S},
\]
where $r^\mu(s)\coloneqq\sum_{s'\in\mathcal{S}}P^\mu(s,s')\,r(s,s')$.

When $n$ becomes large, solving the Bellman expectation equation via matrix inversion becomes infeasible. Instead, we consider Temporal-Difference (TD) learning with linear function approximation \citep{sutton1988learning}. Let $\bphi:\mathcal{S}\to\mathbb{R}^d$ be a fixed feature mapping with $d\ll n$, and let $\btheta\in\mathbb{R}^d$. We approximate $\bV^\mu$ by $\bV_{\btheta}(s)\coloneqq\btheta^\top\bphi(s)$, or in vector form $\bV_{\btheta}=\bPhi\btheta$, where the feature matrix $\bPhi\in\mathbb{R}^{n\times d}$ has row $\bphi(s)^\top$ corresponding to state $s$.

\paragraph{Structural assumption.}
Throughout the paper, we impose the following standard condition on the induced MRP and the feature matrix.
\begin{assumption}[Ergodic induced MRP and full-rank features]
\label{ass:mrp-ergodic-full-rank}
The transition matrix $\bP^\mu$ is irreducible and aperiodic, and the feature matrix $\bPhi$ has full column rank.
\end{assumption}

Since $\bP^\mu$ is finite, irreducible, and aperiodic, it admits a unique stationary distribution $\pi$. We write $\bD\coloneqq\mathrm{diag}(\pi)$ and use the value-space norm $\norm{\bv}_{\bD}\coloneqq\sqrt{\bv^\top\bD\bv}$ for $\bv\in\mathbb{R}^n$. We also define $\Sigma\coloneqq\bPhi^\top\bD\bPhi$. Under Assumption~\ref{ass:mrp-ergodic-full-rank}, $\pi(s)>0$ for every $s\in\mathcal{S}$, and hence $\Sigma\succ0$.

For any $\btheta_0\in\mathbb{R}^d$, the TD($0$) algorithm is defined for $t\ge0$ by
\[
\btheta_{t+1}=\btheta_t+\eta_t\bg(\btheta_t,Z_t),\qquad Z_t\coloneqq(s_t,s_{t+1}),
\]
where $\eta_t>0$ is the stepsize and
\[
\bg(\btheta,Z_t)\coloneqq\bigl(r(s_t,s_{t+1})+\gamma\bphi(s_{t+1})^\top\btheta-\bphi(s_t)^\top\btheta\bigr)\bphi(s_t).
\]
When the dependence on randomness and iterates is clear, we write $\bg_t\coloneqq\bg(\btheta_t,Z_t)$.

The TD($0$) update can equivalently be written as $\btheta_{t+1}=\btheta_t+\eta_t(\bb_{Z_t}-\bA_{Z_t}\btheta_t)$, where
\[
\bA_{Z_t}\coloneqq\bphi(s_t)\bigl(\bphi(s_t)-\gamma\bphi(s_{t+1})\bigr)^\top\in\mathbb{R}^{d\times d},
\qquad
\bb_{Z_t}\coloneqq r(s_t,s_{t+1})\bphi(s_t)\in\mathbb{R}^d.
\]
Let $\mathcal{Z}\coloneqq\{(i,j)\in\mathcal{S}\times\mathcal{S}:P^\mu(i,j)>0\}$ be the state space of the transition chain $Z_t=(s_t,s_{t+1})$. Its Markov kernel is
\[
P_Z\bigl((i,j),(j',k)\bigr)=
\begin{cases}
P^\mu(j,k), & j'=j,\\
0, & j'\neq j,
\end{cases}
\qquad (i,j),(j',k)\in\mathcal{Z},
\]
and its stationary distribution is $\pi_Z(i,j)\coloneqq\pi(i)P^\mu(i,j)$ for $(i,j)\in\mathcal{Z}$. Define the population TD matrix and vector by $\bA\coloneqq\E_{Z\sim\pi_Z}[\bA_Z]$ and $\bb\coloneqq\E_{Z\sim\pi_Z}[\bb_Z]$.

For $z\in\mathcal{Z}$, define
\[
\bxi(z)\coloneqq \bb_z-\bA_z\btheta^*,
\qquad
\bdelta(z)\coloneqq \bA-\bA_z.
\]
When $z=Z_t$, we write
\[
\bxi_t\coloneqq\bxi(Z_t),
\qquad
\bdelta_t\coloneqq\bdelta(Z_t).
\]

As shown in Lemma~\ref{lem:pos-stable} below, Assumption~\ref{ass:mrp-ergodic-full-rank} implies that $\bA$ is nonsingular. Hence the expected TD linear system $\bA\btheta^*=\bb$ has a unique solution $\btheta^*\in\mathbb{R}^d$. It is known that the corresponding value approximation $\bV_{\btheta^*}=\bPhi\btheta^*$ satisfies the projected Bellman equation \citep{tsitsiklis1996analysis}
\[
\bPhi\btheta^*=\Pi_{\bD}T^\mu(\bPhi\btheta^*),
\]
where $\Pi_{\bD}$ is the orthogonal projection operator onto the subspace $\{\bPhi\bx:\bx\in\mathbb{R}^d\}$ with respect to $\norm{\cdot}_{\bD}$. Moreover,
\[
\norm{\bV^\mu-\bV_{\btheta^*}}_{\bD}
\le
\frac{1}{1-\gamma}\norm{\bV^\mu-\Pi_{\bD}\bV^\mu}_{\bD}.
\]

The \emph{mean-path TD update} is defined as
\[
\bar\bg(\btheta)\coloneqq
\E_{Z\sim\pi_Z}\bigl[\bigl(r(s,s')+\gamma\bphi(s')^\top\btheta-\bphi(s)^\top\btheta\bigr)\bphi(s)\bigr].
\]
Using $\bA\btheta^*=\bb$, the mean-path TD update is linear in $\be\coloneqq\btheta-\btheta^*$ for all $\btheta\in\mathbb{R}^d$:
\begin{equation}
\label{eq:bar-g-linear}
\bar\bg(\btheta)=\bb-\bA\btheta=-\bA(\btheta-\btheta^*),\qquad \bar\bg(\btheta^*)=\boldsymbol{0}.
\end{equation}

To characterize the convergence of TD iterates $\btheta_t$ to the TD fixed point $\btheta^*$, we introduce the Dirichlet semi-norm $\|\bx\|_{\mathrm{Dir}}^2\coloneqq\bx^\top\bL_{\mathrm{Dir}}\bx$, where
\[
\bL_{\mathrm{Dir}}\coloneqq\bD-\tfrac12\bigl(\bD\bP^\mu+(\bP^\mu)^\top\bD\bigr).
\]
By Lemma~\ref{lem:L-dir-properties}, $\|\cdot\|_{\mathrm{Dir}}$ is indeed a semi-norm since $\bL_{\mathrm{Dir}}$ is positive semidefinite. The potential function $f$ that we study \citep{ollivier2018approximate,liu2021temporal} is defined as
\begin{equation}
\label{eq:dirichlet}
f(\btheta)\coloneqq(1-\gamma)\,\|\bV_{\btheta}-\bV_{\btheta^*}\|_{\bD}^2+\gamma\,\|\bV_{\btheta}-\bV_{\btheta^*}\|_{\mathrm{Dir}}^2.
\end{equation}
In particular, for $\be=\btheta-\btheta^*$, since $\be^\top\bA\be=\be^\top\bA^\top\be$, a direct calculation yields
\begin{equation}
\label{eq:f-quadratic}
f(\btheta)-f(\btheta^*)=\tfrac12\be^\top(\bA+\bA^\top)\be=\be^\top\bA\be.
\end{equation}
Moreover, using~\eqref{eq:bar-g-linear}, we have
\begin{equation}
\label{eq:dirichlet-monotone}
\langle\bar\bg(\btheta),\btheta^*-\btheta\rangle
=\langle-\bA\be,-\be\rangle
=\be^\top\bA\be
=f(\btheta)-f(\btheta^*).
\end{equation}
Since $f(\btheta)-f(\btheta^*)\ge0$ for $\btheta\neq\btheta^*$, the mean-path update direction $\bar\bg(\btheta)$ is aligned with the direction $\btheta^*-\btheta$. In particular, TD($0$) under i.i.d.\ sampling from $\pi_Z$ acts as a descent method for $f$. Additionally, any convergence guarantee on $f$ can be translated to a guarantee in the projected value norm:
\begin{equation}
\label{eq:potential_transform}
\|\btheta-\btheta^*\|_{\bSigma}^2=\|\bV_{\btheta}-\bV_{\btheta^*}\|_{\bD}^2\le\frac{f(\btheta)}{1-\gamma}.
\end{equation}

\paragraph{Constants and structural consequences.}
We next collect the constants and matrix properties used in the high-probability analysis. The boundedness constants are fixed by the finite state space, while the mixing and curvature constants follow from Assumption~\ref{ass:mrp-ergodic-full-rank}.

\begin{lemma}[Feature and reward bounds]
\phantomsection\label{lem:bounded}
There exist finite constants $\phi_\infty,r_\infty<\infty$ such that $\|\bphi(s)\|\le\phi_\infty$ for all $s\in\mathcal{S}$ and $|r(s,s')|\le r_\infty$ for all $s,s'\in\mathcal{S}$. For $z=(s,s')\in\mathcal{Z}$, set
\[
\varepsilon_z\coloneqq r(s,s')+\gamma\bphi(s')^\top\btheta^*-\bphi(s)^\top\btheta^*.
\]
Then $\bxi(z)=\varepsilon_z\bphi(s)$ and
\[
\|\bxi(z)\|\le r_\infty\phi_\infty+2\phi_\infty^2\|\btheta^*\|.
\]
\end{lemma}

\begin{lemma}[Geometric mixing of $(Z_t)$]
\label{lem:mixing}
The transition chain $Z_t=(s_t,s_{t+1})$ has stationary distribution $\pi_Z$. Moreover, there exists a finite constant $\mixing\ge1$ such that
\[
\sup_{z\in\mathcal{Z}}\|P_Z^k(z,\cdot)-\pi_Z\|_{\mathrm{TV}}
\le
4\cdot2^{-k/\mixing},
\qquad k\ge0,
\]
where $P_Z^k(z,z')\coloneqq\Pr(Z_k=z'\mid Z_0=z)$.
\end{lemma}

\begin{lemma}[Curvature of the TD matrix]
\label{lem:pos-stable}
The symmetric part of $\bA$ is positive definite. More precisely,
\[
\frac{\bA+\bA^\top}{2}
=(1-\gamma)\bPhi^\top\bD\bPhi+\gamma\bPhi^\top\bL_{\mathrm{Dir}}\bPhi
\succeq
(1-\gamma)\lambda_{\min}(\bPhi^\top\bD\bPhi)\bI_d.
\]
Consequently, one may take $\omega=(1-\gamma)\lambda_{\min}(\bPhi^\top\bD\bPhi)>0$, and
\[
f(\btheta)-f(\btheta^*)
=\be^\top\frac{\bA+\bA^\top}{2}\be
\ge
\omega\|\be\|^2
=\omega\|\btheta-\btheta^*\|^2.
\]
Thus $f$ is $\omega$ star-strongly convex around $\btheta^*$.
\end{lemma}

The first lemma is immediate from the finiteness of $\mathcal{S}$ and the triangle inequality. Lemma~\ref{lem:mixing} is the standard geometric mixing bound for finite irreducible and aperiodic Markov chains; see, e.g., \citet[Theorem~4.9]{levin2017markov}. For Lemma~\ref{lem:pos-stable}, expand $\bA=\bPhi^\top\bD(\bI-\gamma\bP^\mu)\bPhi$ and symmetrize. Then Lemma~\ref{lem:L-dir-properties} gives $\bL_{\mathrm{Dir}}\succeq0$, while Assumption~\ref{ass:mrp-ergodic-full-rank} gives $\bPhi^\top\bD\bPhi\succ0$.

\begin{remark}
For Theorem~\ref{thm:bounded-iterates-main}, the curvature conclusion in Lemma~\ref{lem:pos-stable} is not needed. The bounded-iterates argument only requires a sample-path independent reference point $\btheta^*$ satisfying $\bA\btheta^*=\bb$; in more singular settings, this may be replaced by an appropriate reference set. In the discounted full-rank MRP setting above, the reference point exists uniquely. If $\gamma=1$ or if $\bPhi$ is not full column rank, its existence requires additional regularity conditions. Whenever such a reference point exists, the robust $\widetilde{\mathcal{O}}(1/\sqrt{T})$ rate continues to provide control in terms of the Dirichlet semi-norm.
\end{remark}

\paragraph{Poisson Equation.}
A critical tool for handling Markovian noise is the Poisson equation associated with a bounded measurable function $h:\mathcal{Z}\to\mathbb{R}^d$. Define the stationary mean $\pi_Z(h)\coloneqq\E_{Z\sim\pi_Z}[h(Z)]$ and let $\tilde h\coloneqq h-\pi_Z(h)$ be the centered function. The Poisson equation associated with $P_Z$ is
\begin{equation}
\label{eq:poisson_in_setting}
u-P_Zu=\tilde h,
\qquad\text{i.e.,}\qquad
u(z)-\sum_{z'\in\mathcal{Z}}P_Z(z,z')\,u(z')=\tilde h(z),\qquad z\in\mathcal{Z}.
\end{equation}
In Appendix~\ref{sec:poisson-toolkit}, we will prove that this Poisson equation admits a solution $u^*:\mathcal{Z}\to\mathbb{R}^d$ such that $\|u^*\|_\infty\le16\mixing\|h\|_\infty$, where $\|h\|_\infty\coloneqq\sup_{z\in\mathcal{Z}}\|h(z)\|$.

\section{Main Results}
\label{sec:main-results}

In this section, we present our main results. We first establish a high-probability uniform bound on the iterates, which is a key step in our convergence analysis. We then derive a simultaneous high-probability convergence rate that is both curvature-free/robust and curvature-aware/fast convergence rate for the Polyak--Ruppert averaged iterate.

\subsection{High-Probability Boundedness of the Iterates of Unprojected TD($0$)}

Our first result shows that, even without any projection, the iterates of TD(0) remain bounded by a constant multiple of
$\max\{\|\btheta_0-\btheta^*\|,\|\btheta^*\|,\frac {r_\infty}{\phi_\infty}\}$
under a simple stepsize schedule that exploits the update structure of TD($0$).

\begin{theorem}[High-probability bounded iterates]
\label{thm:bounded-iterates-main}
Under Assumption~\ref{ass:mrp-ergodic-full-rank}, let
$\phi_\infty,r_\infty$ and $\mixing$ be the constants introduced in
Lemmas~\ref{lem:bounded} and~\ref{lem:mixing}.  Consider the following stepsize schedule:
\begin{equation}
\label{eq:anytime-stepsize-intro}
\eta_t
=
\eta_{\mathrm{base}} a_t,
\quad \text{where} \quad
\eta_{\mathrm{base}} := \frac{1}{c\mixing \,\phi_\infty^2},
\end{equation}
for some numerical constant $c>0$, and a non-increasing positive sequence $(a_t)$ such that $\sum_{t=0}^{\infty} a_t^2$ is finite and $a_0\leq 1$. Fix any $\delta\in(0,1)$ and let
\begin{equation}
\label{eq:rbase}
\rbase\coloneqq \max \left\{\norm{\btheta_0-\btheta^*},\norm{\btheta^*}, \frac{r_\infty}{\phi_\infty}\right\}\,.
\end{equation}
Define\footnote{We keep track of all numerical constants because future work may focus on sharpening the minimal value of $c$.} $A_1(\delta)
\coloneqq 1536\sqrt{\sum_{t=0}^{\infty} a_t^2}\sqrt{2\log\frac{2}{\delta}}+2304$,
$A_2\coloneqq 2706\sum_{t=0}^{\infty} a_t^2$,
\begin{align}
\label{eq:cmin-rho}
c_{\min}(\delta)
&\coloneqq \frac{A_1(\delta)+\sqrt{A^2_1(\delta)+4A_2}}{2},
\qquad \text{and} \qquad
\rho \coloneqq \frac{2c}{\sqrt{c^2-A_1(\delta)c-A_2}}\,.
\end{align}
Then, provided that $c>c_{\min}(\delta)$, with probability at least $1-\delta$, we have
\[
\sup_{t\geq 0} \ \norm{\btheta_t}_2\leq \rho \rbase\,.
\]
\end{theorem}

Note that a simple choice of $a_t$ that satisfies the assumptions of the theorem is $a_t = \frac{1}{\ln(t+3)\sqrt{t+1}}$.

We provide the complete proof in Appendix~\ref{app:proof-bounded-iterates-main} and a proof sketch below.

\paragraph{Proof Sketch.}
Fix $\delta>0$ and suppress the $\delta$-dependence in the concentration bounds for simplicity. Starting from the definition of the TD($0$) update, we have
\begin{equation}
    \label{eq:bound_of_g}
\norm{\bg_k}
=
\norm{\bigl(r_k + \gamma \langle\bphi(s_{k+1}), \btheta_k\rangle
      - \langle \bphi(s_k), \btheta_k\rangle\bigr)\bphi(s_k)}
\leq |r_k|\norm{\bphi(s_k)} + |1+\gamma|\,\norm{\bphi(s_k)}^2 \norm{\btheta_k}.
\end{equation}
Thus, in each iteration, the update magnitude is at most
$r_\infty\phiinf + 2\phiinfs\norm{\btheta_k}$, which is of the same order as $\norm{\btheta_k}$.

\noindent Next, expand $\norm{\btheta_t-\btheta^*}^2$ using the TD($0$) update and sum from $k=1$ to $t$, to obtain
\begin{align}
\label{eq:expansion_multiple_steps}
\|\btheta_t-\btheta^*\|^2
 &= \|\btheta_{0}-\btheta^*\|^2
 + \sum_{k=1}^t \eta^2_{k-1}\|\bg_{k-1}\|^2
 + 2\sum_{k=1}^t \eta_{k-1}\langle \bg_{k-1}-\bar \bg(\btheta_{k-1}),\btheta_{k-1}-\btheta^*\rangle\nonumber\\
  &\quad +2\sum_{k=1}^t \eta_{k-1}\langle \bar \bg(\btheta_{k-1}),\btheta_{k-1}-\btheta^*\rangle
    \,.
\end{align}
\noindent To build intuition, suppose for the moment that $\bg_{k}=\bar \bg(\btheta_k)$. Then, equation \eqref{eq:expansion_multiple_steps} simplifies to
\begin{align*}
\|\btheta_t-\btheta^*\|^2- \|\btheta_0-\btheta^*\|^2
 \leq
 \sum_{k=1}^{t}\eta^2_{k-1}\|\bar \bg(\btheta_{k-1})\|^2
 =   \mathcal{O}\!\left(\sum_{k=1}^t\eta^2_{k-1}\left(\norm{\btheta_{k-1}}+\norm{\btheta_{k-1}}^2\right)\right),
\end{align*}
where we use $\Ip{\bar \bg(\btheta_{k-1})}{\btheta_{k-1}-\btheta^*}\leq 0$ from equation \eqref{eq:dirichlet-monotone}.
Hence, $\|\btheta_t-\btheta^*\|^2$ is governed by a term of the form
$\mathcal{O}\!\left(\sum_{k=1}^t\eta^2_{k-1}\bigl(\norm{\btheta_{k-1}}+\norm{\btheta_{k-1}}^2\bigr)\right)$.
Since $(\eta_k)$ is square-summable, if $\norm{\btheta_{k}}$ is bounded by $\rbase$ for all $k\leq t-1$, then $\norm{\btheta_t-\btheta^*}$ is also bounded by a constant multiple of $\rbase$.
Moreover, by the triangle inequality,
$\norm{\btheta_t}\leq \norm{\btheta_t-\btheta^*} + \norm{\btheta^*}$.
In other words, a bound on $\norm{\btheta_{k}}$ for all $k\leq t-1$ implies a bound at time $t$.
This motivates the following induction hypothesis for some $\rho>2$:
\[
\max_{1\leq k\leq t-1} \ \norm{\btheta_{k-1}}\leq \rho \rbase\,.
\]
The core of the proof then becomes to choose the stepsize parameter $c$ (and hence $\rho$) so that
\[
\|\btheta_t\|^2\leq 2\|\btheta_t-\btheta^*\|^2
+ 2\|\btheta^*\|^2
 \leq \rho^2 \rbases,
\]
where we use the induction hypothesis to control $\norm{\btheta_{t}-\btheta^*}$.

\noindent
We now return to the Markovian bias term
$\sum_{k=1}^t \eta_{k-1}\langle \bg_{k-1}-\bar \bg(\btheta_{k-1}),\btheta_{k-1}-\btheta^*\rangle$
in~\eqref{eq:expansion_multiple_steps}.
In general, for $\mathcal{F}_t:=\sigma(Z_0,\dots,Z_t)$, $\Ec[\mathcal{F}_{k-2}]{\bg_{k-1}-\bar \bg(\btheta_{k-1})}$ need not be zero: $\bg_{k-1}$ depends on the fresh randomness in $Z_{k-1}$, and the conditional distribution of $Z_{k-1}$ given $Z_{k-2}$ does not necessarily coincide with the stationary distribution $\pi_Z$.
To control this bias, freeze the past iterate $\btheta_{k-1}$ and define the centered scalar forcing term $h_{k-1}(z)\coloneqq \langle \bg(\btheta_{k-1},z)-\bar \bg(\btheta_{k-1}),\btheta_{k-1}-\btheta^*\rangle$. For this forcing term, consider the Poisson equation on the transition chain $Z_t$:
\begin{equation}
    \label{eq:poisson_equation}
    u_{k-1}:\mathcal{Z}\rightarrow \R,\quad
    u_{k-1}(z)-\Ec[Z_0=z]{u_{k-1}(Z_1)} = h_{k-1}(z)\,.
\end{equation}
By Lemma~\ref{lem:poisson-solution}, the equation \eqref{eq:poisson_equation} is solved by
$u_{k-1}^*(z)\coloneqq \sum_{i=0}^\infty \Ec[Z_0=z]{h_{k-1}(Z_{i})}$, and we also have
\begin{equation}
    \label{eq:poisson_norm_bound}
    \norm{u_{k-1}^*}_{\infty}
    \coloneqq \sup_z \ |u_{k-1}^*(z)|
    \leq 16\mixing \norm{h_{k-1}}_{\infty}
    =\mathcal{O}\left(\mixing\norm{\btheta_{k-1}}^2\right)\,.
\end{equation}
The solution $u^*_{k-1}$ provides a convenient decomposition of
$\sum_{k=1}^t \eta_{k-1} h_{k-1}(Z_{k-1})$
into a martingale term $M_t$ plus a remainder term $R_t$:
\begin{align}
    \sum_{k=1}^t \eta_{k-1}h_{k-1}(Z_{k-1})
      &= \sum_{k=1}^t  \eta_{k-1}\left(
      u^*_{k-1}(Z_{k})- \Ec[Z_{k-1}]{u^*_{k-1}(Z_{k})}
      + u^*_{k-1}(Z_{k-1})-u^*_{k-1}(Z_{k})
      \right)\nonumber \\
      &= M_t+R_t,\label{eq:Poisson_decomposition}
\end{align}
where $\Delta M_k\coloneqq \eta_{k-1} (u^*_{k-1}(Z_{k})- \E[u^*_{k-1}(Z_{k})|Z_{k-1}])$ and $M_t\coloneqq \sum_{k=1}^t \Delta M_k$.
Note that
$\norm{\Delta M_k}_{\infty}=\mathcal{O}(\eta_{k-1}\mixing\norm{\btheta_{k-1}}^2)
= \mathcal{O}(a_{k-1}\norm{\btheta_{k-1}}^2)$ from \eqref{eq:poisson_norm_bound}.
Consequently, the accumulated bounded difference for $M_t$ is controlled by
$\mathcal{O}(\max_{k\leq t} \ \norm{\btheta_{k-1}}^2)$.
Assuming again that $\max_{k\leq t} \ \norm{\btheta_{k-1}}\leq \rho \rbase$, Pinelis' inequality~\citep{pinelis1994optimum} for martingale concentration (Lemma~\ref{lem:pinelis-anytime-R}) yields $|M_t|=\mathcal{O}(\rho^2 \rbases)$ with high probability.

\noindent To control the remainder term
$R_t=\sum_{k=1}^t \eta_{k-1}(u^*_{k-1}(Z_{k-1})-u^*_{k-1}(Z_{k}))$,
we rewrite it (see Lemma~\ref{lem:R-Abel}) as
\[
R_t=\eta_0 u^*_0(Z_0)-\eta_t u^*_t(Z_t)
-
\sum_{k=1}^t(\eta_{k-1}-\eta_k)u^*_{k-1}(Z_k)
-
\sum_{k=1}^t\eta_k\bigl(u^*_{k-1}(Z_k)-u^*_k(Z_k)\bigr)\,.
\]
Since $\norm{u^*_{k}-u^*_{k-1}}_\infty$ is controlled by
$\mathcal{O}(\eta_{k-1}\mixing\norm{\btheta_{k-1}}^2)$
(because $\btheta_k-\btheta_{k-1}=\eta_{k-1}\bg_{k-1}$), the induction hypothesis again implies
$|R_t|=\mathcal{O}(\rho^2 \rbases)$. Combining the bounds for $M_t$ and $R_t$, we obtain
$\sum_{k=1}^t \eta_{k-1} h_{k-1}(Z_{k-1})= \mathcal{O}(\rho^2 \rbases)$, which fits perfectly with our inductive proof, since $\norm{\btheta_t-\btheta^*}^2$ remains the same order $\mathcal{O}(\rho^2 \rbases)$ as in the idealized case $\bg_{k}=\bar \bg(\btheta_k)$.

\subsection{High-Probability Convergence Rate for Polyak--Ruppert Averaging}
Our second result studies the convergence rate of the Polyak--Ruppert (PR) averaged iterate.
Let
\[
S_T
:=
\sum_{t=1}^{T}\eta_{t-1}
\qquad
\bar{\btheta}_T
:=
\frac{1}{S_T}\sum_{t=1}^{T}\eta_{t-1}\btheta_{t-1}\,.
\]
Recall the following potential defined in~\eqref{eq:dirichlet}:
\[
f(\btheta)-f(\btheta^*)
=
(1-\gamma)\,\|\bV_{\btheta}-\bV_{\btheta^*}\|_{\bD}^2
+
\gamma\,\|\bV_{\btheta}-\bV_{\btheta^*}\|_{\mathrm{Dir}}^2\,.
\]
We have the following high-probability guarantees for the PR average $\bar{\btheta}_T$.

\begin{theorem}[High-probability rate for PR averaging]
\label{thm:PR-main}
Under Assumption~\ref{ass:mrp-ergodic-full-rank}, let
$\phi_\infty,r_\infty,\mixing$ and $\omega$ be the constants introduced in
Lemmas~\ref{lem:bounded}--\ref{lem:pos-stable}. Consider the same stepsize schedule $(\eta_t)$ as in \eqref{eq:anytime-stepsize-intro} and the same $\rbase$ in \eqref{eq:rbase}. Also, define $H\coloneqq \sum_{t=0}^\infty\eta_{t}^2$. Fix any $\delta\in(0,1)$ and let
$A_1(\delta)=1536\sqrt{\sum_{t=0}^{\infty} a_t^2}\sqrt{2\log\frac{8}{\delta}}+2304$ and $A_2= 2706\sum_{t=0}^{\infty} a_t^2$.
Define
\[
c_{\min}(\delta)
\coloneqq
\frac{A_1(\delta)+\sqrt{A_1(\delta)^2+4A_2}}{2},
\qquad
\rho
\coloneqq
\frac{2c}{\sqrt{c^2-A_1(\delta)c-A_2}}\,.
\]
Suppose $c> c_{\min}(\delta)$. Then, with probability at least $1-\delta$, we have for all $T\geq 1$,
\[
f(\bar{\btheta}_T)-f(\btheta^*)
\leq \min\left\{\frac{C_{\mathrm{fast}}^2}{\omega S_T^2},\frac{C_{\mathrm{robust}}}{S_T}\right\},
\]
where
\begin{align*}
C_{\mathrm{fast}}&\coloneqq \rho\rbase\left[2+2\mixing\phiinfs\left(264\eta_0
+176\sqrt{H}\sqrt{2\log\frac 8 \delta}\right) + 192\mixing\phi_\infty^4 H\right],\\
C_{\mathrm{robust}}&\coloneqq \rho^2\rbases\left[0.5 + 2\mixing \phiinfs \left(288 \eta_0+192\sqrt{H} \sqrt{2\log\frac{8}{\delta}}\right) +\left(672\mixing+4.5\right)\phiinf^4 H\right]\,.
\end{align*}
\end{theorem}

A choice of the stepsize that satisfies the assumptions of the theorem is (see Appendix~\ref{sec:stepsize})
\begin{equation}
\label{eq:beststepsize}
\eta_t = (c\,\phiinfs\,\mixing)^{-1} a_t, \quad \text{ where } \quad a_t = (\log(t+3)\sqrt{t+1})^{-1}\,.
\end{equation}
The logarithmic correction in $a_t$ is used at the square-summability boundary: it ensures that $\sum_{t=0}^{\infty}\eta_t^2$ is finite while keeping $S_T=\sum_{t=1}^{T}\eta_{t-1}$ within a logarithmic factor of $\sqrt{T}$. Such a choice results in a robust rate of $\widetilde{\mathcal{O}}(\frac{\|\btheta^*\|^2 \mixing \phi_\infty^2}{\sqrt{T}})$ and a fast rate of $\widetilde{\mathcal{O}}(\frac{\|\btheta^*\|^2\mixing^2 \phi_\infty^4}{\omega T})$ with high probability.

In Appendix~\ref{sec:no_tau_mixing}, we also show an alternative stepsize choice that does not require knowledge of $\mixing$, at the cost of a doubly exponential dependence on $\mixing$ in the constants.

\paragraph{Proof Sketch of Theorem~\ref{thm:PR-main}.}
By Theorem~\ref{thm:bounded-iterates-main}, with high probability we have
$\sup_{t\geq 0} \ \norm{\btheta_t}_2 \leq \rho\,\rbase$.
It then follows that
$\sum_{t=1}^T \eta_{t-1}^2 \norm{\bg_{t-1}}^2$
is of order $\mathcal{O}(\rbases)$, using the gradient bound~\eqref{eq:bound_of_g}.
Moreover, the cumulated Markov bias term
$\sum_{t=1}^T \eta_{t-1}\langle \bg_{t-1}-\bar \bg(\btheta_{t-1}),\btheta_{t-1}-\btheta^*\rangle$
is also of order $\mathcal{O}(\rbases)$, by applying the Poisson equation and Pinelis' inequality.

\noindent Rearranging~\eqref{eq:expansion_multiple_steps}, we obtain
\begin{align*}
2\sum_{t=1}^T \eta_{t-1}\langle \bar \bg(\btheta_{t-1}),\btheta^*-\btheta_{t-1}\rangle
&= \|\btheta_{0}-\btheta^*\|^2-\|\btheta_T-\btheta^*\|^2\\
&\quad + \sum_{t=1}^T \eta_{t-1}^2\|\bg_{t-1}\|^2
+ 2\sum_{t=1}^T \eta_{t-1}\langle \bg_{t-1}-\bar \bg(\btheta_{t-1}),\btheta_{t-1}-\btheta^*\rangle\\
&= \mathcal{O}(\rbases)\,.
\end{align*}
The robust rate then follows from the convexity of $f$ together with the equation~\eqref{eq:dirichlet-monotone}:
\begin{align*}
f(\bar{\btheta}_T)-f(\btheta^*)
&\leq
\frac{1}{\sum_{t=1}^{T}\eta_{t-1}}
\sum_{t=1}^T \eta_{t-1}
\Ip{\bar \bg(\btheta_{t-1})}{\btheta^*-\btheta_{t-1}}
= \mathcal{O}\left(\frac{\rbases}{S_T}\right)\,.
\end{align*}

To derive the fast rate, recall that the TD update can be written as
$
\bg(\btheta_{t-1},Z_{t-1})
= -\bA_{Z_{t-1}}\btheta_{t-1}+\bb_{Z_{t-1}}\,.
$
Then the centered error $\be_t:=\btheta_t-\btheta^*$ can be shown (see \eqref{eq:e-recursion}) to evolve as
\begin{equation}
\label{eq:LSA-decomposition}
\be_t
=
(\bI_d-\eta_{t-1}\bA)\be_{t-1}
+ \eta_{t-1}\bxi_{t-1}
+ \eta_{t-1}\bdelta_{t-1}\be_{t-1}\,.
\end{equation}
Rearranging yields
\[
\bA\be_{t-1}
=
\frac{\be_{t-1}-\be_t}{\eta_{t-1}}
+ \bxi_{t-1}
+ \bdelta_{t-1}\be_{t-1}\,.
\]
Define
$\bar{\be}_T := \frac{1}{S_T}\sum_{t=1}^{T}\eta_{t-1}\be_{t-1}$.
Then $\bA \bar{\be}_T = I_1 + I_2 + I_3$, where
\[
I_1:=\frac{1}{S_T}\sum_{t=1}^{T}(\be_{t-1}-\be_t),
\qquad
I_2:=\frac{1}{S_T}\sum_{t=1}^{T}\eta_{t-1}\bxi_{t-1},
\qquad
I_3:=\frac{1}{S_T}\sum_{t=1}^{T}\eta_{t-1}\bdelta_{t-1}\be_{t-1}\,.
\]
By Lemma~\ref{lem:pos-stable}, we have $(\bA+\bA^\top)/2\succeq \omega \bI_d$.
Combining this with the equation~\eqref{eq:dirichlet-monotone} again, we obtain
\[
f(\bar{\btheta}_T)-f(\btheta^*)
=
\langle \bA\bar{\be}_T,\bar{\be}_T\rangle
\le
\|\bA\bar{\be}_T\|\,\|\bar{\be}_T\|
\le
\frac{1}{\omega}\,\|\bA\bar{\be}_T\|^2\,.
\]
The fast rate follows by controlling $\norm{I_1}$, $\norm{I_2}$, and $\norm{I_3}$ separately.
The term $\|I_1\|$ is bounded using the telescoping sum.
For $\|I_2\|$ and $\|I_3\|$, we again apply Lemma~\ref{lem:poisson-solution} with forcing terms $\bxi_{t-1}$ and $\bdelta_{t-1}\be_{t-1}$, respectively.
Therefore, by Pinelis' inequality and the normalization by $S_T$, both $\|I_2\|$ and $\|I_3\|$ are of order $\mathcal{O}(\rbase/S_T)$ with high probability, which completes the proof sketch for the fast rate.

\section{Detailed Comparison with Prior Results}
\label{sec:comparison}

In this section, we compare Theorems~\ref{thm:bounded-iterates-main} and~\ref{thm:PR-main} with prior finite-time results for TD($0$) with linear function approximation. This comparison is delicate because prior papers use slightly different assumptions on the problem and the algorithm. We therefore compare the results on a common footing, while emphasizing which unknown quantities are needed to set the algorithmic hyperparameters.
For simplicity, in this section we assume $\phi_\infty= 1$.

First, we recall our guarantees under the stepsize schedule in~\eqref{eq:beststepsize}: $\eta_t=\frac{1}{c\,\mixing\sqrt{t+1}\log(t+3)}$, where $c$ satisfies the assumptions of Theorem~\ref{thm:bounded-iterates-main}. Under this choice, TD($0$) with Polyak--Ruppert averaging \emph{simultaneously} achieves the following three guarantees:
\begin{itemize}
\item bounded iterates, independent of $\omega$,
\item the $\omega$-robust rate
$\mathcal{O}\Big(\frac{\mixing \rbases \log T\sqrt{\log(1/\delta)}}{\sqrt{T}}\Big)$,
\item the fast $\omega$-aware rate
$\mathcal{O}\Big(\frac{\mixing^2 \rbases \log^2 T\log(1/\delta)}{\omega T}\Big)$,
with probability at least $1-\delta$.
\end{itemize}
Importantly, all guarantees use a single stepsize schedule that is independent of the curvature parameter $\omega$.

\paragraph{Prior work on bounded iterates.}
To the best of our knowledge, our result is the first to show that TD($0$) with Markovian sampling has bounded iterates with high probability without using $\omega$ in the stepsize or in a curvature-based stability argument.

The closest result is \citet{lee2025finite}, but they only provide expectation bounds, whose rate we match. While one would expect that it is relatively easy to convert expectation results to high-probability ones for well-behaved random variables, here it requires completely different machinery based on the Poisson equation.

Our proof proceeds by establishing the boundedness of the iterates via a carefully designed inductive argument. This argument uses only the condition $\Ip{\bar\bg(\btheta)}{\btheta-\btheta^*}\leq 0$. We then control the Markovian bias pathwise using a Poisson-equation-based decomposition. This approach is fundamentally different from the Linear Stochastic Approximation (LSA)-stability-based analyses in \citet{srikant2019finite,pmlr-v125-mou20a,chen2022finite,patil2023finite,Li2024SampleComplexities,samsonov2024improved,Durmus2025} that result in a bound on the norm of the iterates that depends on $\omega$, hence incompatible with the robust rate.

Our approach is also very different from using projections, that is, projecting $\btheta_k$ onto the ball $\mathcal{B}(\bzero, \norm{\btheta^*})$~\citep[see, e.g.,][]{bhandari2018finite,liu2021temporal,sun2021adaptive,prashanthConcentrationBoundsTemporal2021,patil2023finite}.

Such projections play an analytic rather than algorithmic role: they guarantee that $|\bg(\btheta,\cdot)|_\infty$ is always bounded by a constant depending on $\norm{\btheta^*}$,
so that the following terms can be controlled easily using concentration inequalities:
\[
\sum_k \eta_k^2\|\bg_k\|^2,
\qquad
\sum_k \eta_k\Ip{\bg_k-\bar\bg(\btheta_k)}{\btheta_k-\btheta^*}.
\]
These terms scale with $\norm{\btheta^*}^2$ instead of an unknown bound on $\norm{\btheta_k}^2$. We remark that this kind of modification of TD($0$) is not used in practice, and it would require an additional hyperparameter. Moreover, simply bounding $\norm{\btheta^*}$ with the loose bound~\citep[Lemma 7]{bhandari2018finite} $\norm{\btheta^*}\leq \frac{2r_\infty}{\sqrt{\omega}(1-\lambda)}$ once again breaks the $\omega$-independent result required for the robust rate.

\paragraph{Prior work on robust rates.}
To the best of our knowledge, ours is the first result to achieve the robust rate
$\widetilde{\mathcal O}(\frac{\norm{\btheta^*}^2\mixing}{\sqrt{T}})$ with high probability, without any algorithmic modification and with a simple stepsize schedule independent of $\omega$.
Once one considers the implicit dependence of $T$ on $\mixing$ in some prior work, we match the rate in expectation from prior results~\citep[see, e.g.,][]{bhandari2018finite, liu2021temporal, lee2025finite}, up to polylog factors. On the other hand, we require knowledge of $\mixing$ to set the stepsizes. Instead, prior work, either implicitly or explicitly, requires the number of iterations $T$ to be large enough with respect to some function of $\mixing$ for the rate to be $\widetilde{\mathcal O}(\frac{1}{\sqrt{T}})$.

Our analysis is fundamentally different from previous ones,with the notable exception of the expectation result of \citet{lee2025finite} that we already discussed.
Indeed, the stability-based approach commonly used in the literature, which relies on \eqref{eq:LSA-decomposition} to unroll $\btheta_t-\btheta^*$, would typically result in a bound on $\norm{\btheta_t-\btheta^*}$ that scales like $\sum_{i=1}^t (1-\lambda_{\min}\eta)^i$ using inequalities similar to \eqref{eq:exp-stability}, which only controls $\norm{\btheta_t}$ at the scale $\mathcal{O}(1/\lambda_{\min})$. This dependence is incompatible with the $\omega$-independent bounded-iterates guarantee in Theorem~\ref{thm:bounded-iterates-main}.

\paragraph{Prior work on fast rates.}

The most relevant prior works on fast rates are \citet{samsonov2024improved,Durmus2025}.\footnote{As observed by \citet[p.~7]{Durmus2025}, the results in \citet{pmlr-v125-mou20a} are based on restrictive assumptions that do not allow a fair comparison with ours and previous results, so we omit them here.} These papers view TD($0$) with linear function approximation as a special case of LSA. They prove that there exist constants $a>0$, $\varkappa_p>0$, and $\alpha_{p,\infty}>0$ (depending on $p$ and possibly on $\mixing^{-1}$) such that $p\alpha_{p,\infty}\leq 1/2$ and TD($0$) satisfies the following exponential stability condition: for any moment order $p>0$ and any stepsize $\alpha\in(0,\alpha_{p,\infty})$
\begin{equation}
\label{eq:exp-stability}
\E^{1/p}\left[\norm{\Gamma^{\alpha}_{m:n}u}^p\right]\leq \varkappa_p(1-\alpha a)^n\norm{u},\quad \text{for any }u\in \R^d,\quad 1\leq m\leq n\,,
\end{equation}
where $\Gamma^{\alpha}_{m:n}=\prod_{k=m}^n(\bI-\alpha\bA_{Z_k})$ is a product of random matrices. This condition allows one to unroll the recursion and decompose $\btheta_t-\btheta^*$ as
\begin{align*}
\btheta_t-\btheta^*
&= \left(\bI-\alpha\bA_{Z_t}\right)(\btheta_{t-1}-\btheta^*)+\alpha\bxi_{Z_t}= \Gamma^{\alpha}_{1:t}(\btheta_0-\btheta^*)+\alpha\sum_{i=1}^t \Gamma^{\alpha}_{i+1:t}\bxi_{Z_i}\,.
\end{align*}
Then, one can use $p$-th moment bounds on $\Gamma^{\alpha}_{m:n}$ to control $\|\btheta_t-\btheta^*\|_{\bSigma}$.

Under this approach, and recalling that $\omega\geq (1-\gamma)\lambda_{\min}(\bPhi^\top\bD\bPhi)$, \citet[Theorem~6]{samsonov2024improved} shows that, with high probability, a data-dropping modification of TD($0$)~\citep{nagaraj2020least} achieves the rate
$\mathcal{O}(\frac{\mixing \|\btheta^*\|^2\log^2(T/\delta)}{(1-\gamma)^2\lambda^2_{\min}T})$ for $\|\bar{\btheta}_T-\btheta^*\|^2_{\bSigma}$.
Their bound also includes an additional exponentially decaying term that depends on $\|\btheta_0-\btheta^*\|^2$.
The fixed stepsize $\alpha$ in \citet{samsonov2024improved} is $\omega$-agnostic, but the data-dropping strategy requires knowledge of $\mixing$, and it is not commonly used in practice.

In parallel, \citet[Corollary~2]{Durmus2025} considers a varying stepsize $\eta_t=\mathcal{O}(t^{-2/3})$ that depends on $\lambda_{\min}$ and $\mixing$ and obtains a better high-probability guarantee in which $\|\bar{\btheta}_T-\btheta^*\|^2_{\bSigma}$ decays at rate
$\mathcal{O}(\frac{\mixing\|\btheta^*\|^2 \log(1/\delta)}{(1-\gamma)^2\lambda_{\min}T})$, plus an exponentially decaying term depending on $\|\btheta_0-\btheta^*\|^2$.

To compare these results with our bounds, we can use equation~\eqref{eq:potential_transform} to translate our bound on $f$ into a bound on $\|\bar{\btheta}_T-\btheta^*\|^2_{\bSigma}$. Hence, we obtain the fast rate
$\widetilde{\mathcal{O}}(\frac{\mixing^2\rbases}{(1-\gamma)^2\lambda_{\min}T})$
in Theorem~\ref{thm:PR-main}. Up to logarithmic factors, this rate matches the best-known high-probability results in the Markovian sampling regime~\citep{samsonov2024improved,Durmus2025} and the i.i.d.\ stationary sampling regime~\citep{Li2024SampleComplexities} for all relevant non-mixing quantities. Indeed, our $\rbases$ has the same order of magnitude as $|\btheta^*|^2$ that appears in the other bounds.
However, our $\mixing^2$ dependence is worse by one factor of $\mixing$ than the linear $\mixing$ dependence obtained in the Markovian bounds of \citet{samsonov2024improved,Durmus2025}.
In addition, our initial-error term $\|\btheta_0-\btheta^*\|^2$ decays at a $1/T$ rate rather than exponentially, since we avoid relying on a contraction-based argument in our analysis.

The worse dependence on $\mixing$ could be a side effect of our Poisson-equation-based analysis.
On the other hand, it could also be due to the fact that we do not use a data-dropping method~\citep{samsonov2024improved} or knowledge of $\lambda_{\min}$~\citep{Durmus2025}.

More generally, it is unclear whether any of the above results are optimal in their respective settings. To the best of our knowledge, the closest lower bound is $\Omega(\frac{\|\btheta^*\|^2}{(1-\gamma)\lambda_{\min}T})$ due to \citet[Theorem~2]{Li2024SampleComplexities}, which considers TD($0$) with linear function approximation under the easier setting of i.i.d.\ sampling. Thus, when data dropping is not allowed, projections are not used, and $\lambda_{\min}$ is unknown, the optimal dependence of our rate on $\mixing$ and $1-\gamma$ remains open.

A reason one might expect such a dependence on $\mixing$ in the convergence rate or stepsize is that the TD fixed point $\btheta^*$ is defined with respect to the stationary distribution, and a slowly mixing Markov chain can require more samples to accurately estimate stationary expectations~\citep{nagaraj2020least}.
At the same time, tabular results show that one should not expect a multiplicative mixing-time dependence to be necessary in all settings: \citet{li2020sample} show that, for asynchronous Q-learning and hence tabular TD, the leading statistical term can match the i.i.d.\ sampling setting, with only an additive transient $\mixing$ cost.
In the linear function approximation setting, existing results still account for the mixing time in various ways: stepsize conditions that explicitly depend on $\mixing$ appear in \citep{srikant2019finite, mitra2025simple, Durmus2025}; data-dropping approaches that retain only one out of $\mathcal{O}(\mixing)$ samples are studied in \citep{patil2023finite, samsonov2024improved}; and several results require the total sample budget $T$ to be sufficiently large relative to $\mixing$~\citep{lee2025finite, li2025towards}.

\section*{Acknowledgments}
We thank Chi-Jen Lu, Li Chen, and our anonymous reviewers for their insightful comments. We also acknowledge the use of large language models, including GPT and Gemini, for editorial assistance and for suggesting that Poisson-equation-based methods may provide a useful way to handle randomness in our analysis. The authors independently verified, developed, and finalized all mathematical arguments, results, and references.

\bibliographystyle{plainnat}
\bibliography{bib_db}

\clearpage
\appendix

\section{Poisson Toolkit for Markov Noise}
\label{sec:poisson-toolkit}

In this section, we develop a Poisson-equation toolkit for Markov
concentration. Throughout, let $(Z_t)_{t\ge0}$ be the Markov chain on
\[
\mathcal{Z}:=\{(i,j)\in\mathcal{S}\times\mathcal{S}:P^\mu(i,j)>0\}
\]
with kernel $P_Z$ and stationary distribution $\pi_Z$, as defined in
Section~\ref{sec:assumptions}. The state space $\mathcal{Z}$ is finite.

\subsection{Poisson Equation under Geometric Mixing}

We start with a standard Poisson equation lemma for geometrically mixing chains in $\R^d$ \citep[Theorem~17.4.3]{meyn2012markov}.

\begin{lemma}[Poisson solution under geometric mixing]
\label{lem:poisson-solution}
Let $h:\mathcal{Z}\to \R^d$ be a bounded measurable function with
stationary mean zero, i.e., $\E_{\pi_Z}[h(Z)]=0$. Then, the Poisson equation
\begin{equation}
\label{eq:poisson}
u - P_Z u = h,
\qquad\text{on }\mathcal{Z},
\end{equation}
where $(P_Zu)(z):=\E[u(Z_1)\mid Z_0=z]$, admits a solution
\[
u^*(z) := \sum_{k=0}^\infty (P_Z^k h)(z),
\]
where $(P_Z^k h)(z):=\E[h(Z_k)\mid Z_0=z]$.
Moreover,
\[
\|u^*\|_\infty
\leq
16\mixing \|h\|_\infty,
\qquad \text{where }
\|h\|_\infty
:=
\sup_{z\in\mathcal{Z}} \ \|h(z)\|\,.
\]
\end{lemma}

\begin{proof}
Define $u^*(z):=\sum_{k=0}^\infty (P_Z^k h)(z)$. Since
$\E_{\pi_Z}[h(Z)]=0$, for every $k\ge0$ and every
$z\in\mathcal{Z}$,
\[
(P_Z^k h)(z)
= \int h(z')\{P_Z^k(z,dz')-\pi_Z(dz')\}\,.
\]
Using the geometric mixing bound from Lemma~\ref{lem:mixing}, we have for all $k\ge0$,
\[
\begin{aligned}
\|(P_Z^k h)\|_\infty
&= \sup_{z\in\mathcal{Z}} \ \left\|\E\bigl[h(Z_k)\mid Z_0=z\bigr]\right\|
\le \sup_{z\in\mathcal{Z}} \ 2\|h\|_\infty
\|P_Z^k(z,\cdot)-\pi_Z\|_{\mathrm{TV}}
\le 8\cdot 2^{-k/\mixing}\|h\|_\infty\,.
\end{aligned}
\]
Thus, the series defining $u^*$ converges absolutely and uniformly with
\[
\|u^*\|_\infty
\le \sum_{k=0}^\infty\|(P_Z^k h)\|_\infty
\le \frac{8}{1-2^{-1/\mixing}}\|h\|_\infty
\coloneqq C(\mixing)\|h\|_\infty\,.
\]
Moreover,
\[
(P_Z u^*)(z)
= \sum_{k=0}^\infty P_Z^{k+1} h(z),
\]
so
\[
u^*(z)-(P_Z u^*)(z)
= (P_Z^0 h)(z)
= h(z),
\]
that is, $u^*$ solves~\eqref{eq:poisson}.

\noindent Let $x = 1/\mixing$. We scale the function $C(\mixing)$ by $1/\mixing$ and consider the new function $\tilde{C}(x)$:
\[
\tilde C(x)
:= \frac{1}{\mixing} C\left(\frac{1}{x}\right)
= x \cdot \frac{8 \cdot 2^x}{2^x - 1}
= \frac{8x}{1 - 2^{-x}}\,.
\]
We claim that $\tilde C(x)$ is monotonically increasing on $(0, 1]$. Let $t = x \ln 2$. Then $\tilde C(x)$ can be written as $\frac{8}{\ln 2} \frac{t}{1 - e^{-t}}$. Consider the derivative of $f(t) = \frac{t}{1 - e^{-t}}$ for $t > 0$:
\[
f'(t)
= \frac{1(1 - e^{-t}) - t(e^{-t})}{(1 - e^{-t})^2}
= \frac{1 - (1+t)e^{-t}}{(1 - e^{-t})^2}\,.
\]
Using the elementary inequality $e^t > 1+t$ for $t > 0$, we have $1 > (1+t)e^{-t}$, which implies $f'(t) > 0$. Thus, $\tilde C(x)$ is strictly increasing on $(0, 1]$.
Consequently, the maximum value of $\tilde C(x)$ on the interval $(0, 1]$ occurs at $x = 1$:
\[
\tilde C(x)
\le \tilde C(1)
= \frac{8\cdot 1}{1 - 2^{-1}}
= \frac{8}{0.5}
= 16\,.
\]
Therefore, we have $C(\mixing) \leq 16\mixing$, which completes the proof.
\end{proof}

\subsection{Bounds for TD-type Poisson solutions}
\label{sec:td-poisson-bounds}

We now specialize Lemma~\ref{lem:poisson-solution} to the functions arising from TD. For $z=(s,s^\prime)\in\mathcal{Z}$, define the scalar function
\[
h_{{\btheta},{\btheta}-{\btheta}^*}(z)
:= \bigl\langle \bg({\btheta},z)-\bar \bg({\btheta}),\,{\btheta}-{\btheta}^*\bigr\rangle,
\]
and the vector-valued functions
\begin{align*}
h(z)
&:= \varepsilon_z\bphi(s)\in\mathbb{R}^d,
\qquad \text{where we recall }
\varepsilon_z=r(s,s^\prime)+\gamma\bphi(s^\prime)^\top{\btheta}^*
-\bphi(s)^\top{\btheta}^*,\\
 h_{\btheta}(z)
&:= \bdelta_z\bigl({\btheta}-{\btheta}^*\bigr)\in\mathbb{R}^d,
\qquad \text{where we define }
\bdelta_z\coloneqq \bA-\bA_z\in\mathbb{R}^{d\times d}\,.
\end{align*}
By construction, $h_{{\btheta},{\btheta}-{\btheta}^*}$, $h$, and $h_{\btheta}$ all have
stationary mean zero under $\pi_Z$. We study properties of the corresponding Poisson-equation solutions in the following lemma.

\begin{lemma}[Bounds and Lipschitz properties of TD Poisson equation solutions]
\label{lem:generic-poisson-growth}
In the setting of Section~\ref{sec:assumptions}, let
$\phi_\infty,r_\infty$ and $\mixing$ be the constants introduced in
Lemmas~\ref{lem:bounded} and~\ref{lem:mixing}. Then the following statements hold.
\begin{enumerate}
\item[\emph{(i)}] {\bf Growth bounds.}
For all $z\in\mathcal{Z}$,
\begin{align*}
|h_{{\btheta},{\btheta}-{\btheta}^*}(z)| &\le (2r_\infty\phi_\infty+4\phi_\infty^2\|{\btheta}\|)\,\|{\btheta}-{\btheta}^*\|,\\
\norm{h(z)} &\le r_\infty\phi_\infty + 2\phi_\infty^2\|{\btheta}^*\|,\\
\| h_{\btheta}(z)\| &\le 4\phi_\infty^2 \|{\btheta}-{\btheta}^*\|\,.
\end{align*}

\item[\emph{(ii)}] {\bf Bounded Poisson solutions.}
Let $u^*_{{\btheta},{\btheta}-{\btheta}^*}:\mathcal{Z}\to\mathbb{R}$ be the infinite series solution of the
Poisson equation
\[
u^*_{{\btheta},{\btheta}-{\btheta}^*}-P_Zu^*_{{\btheta},{\btheta}-{\btheta}^*}=h_{{\btheta},{\btheta}-{\btheta}^*},
\]
and let $u^*:\mathcal{Z}\to\mathbb{R}^d$ be the infinite series solution of
\[
u^*-P_Zu^*= h,
\]
and let $u^*_{\btheta}:\mathcal{Z}\to\mathbb{R}^d$ be the infinite series solution of
\[
u^*_{\btheta}-P_Zu^*_{\btheta}= h_{\btheta}\,.
\]
Then,
\begin{align*}
\|u^*_{{\btheta},{\btheta}-{\btheta}^*}\|_\infty &\leq 16\mixing\|h_{{\btheta},{\btheta}-{\btheta}^*}\|_\infty
\le 16\mixing(2r_\infty\phi_\infty+4\phi_\infty^2\|{\btheta}\|)\,\|{\btheta}-{\btheta}^*\|,\\
\|u^*\|_\infty &\leq 16\mixing  \|h\|_\infty \le 16\mixing(r_\infty\phi_\infty + 2\phi_\infty^2\|{\btheta}^*\|),\\
\|u^*_{\btheta}\|_\infty &\leq 16\mixing  \|h_{\btheta}\|_\infty \le 64\mixing\phi_\infty^2 \|{\btheta}-{\btheta}^*\|\,.
\end{align*}
\item[\emph{(iii)}] {\bf Local Lipschitz continuity in ${\btheta}$ and ${\btheta}^\prime$.}
With the definitions in the previous point, we also have
\begin{align*}
\|u^*_{{\btheta},{\btheta}-{\btheta}^*}-u^*_{{\btheta}',{\btheta}'-{\btheta}^*}\|_\infty
&\le 16\mixing\|h_{{\btheta},{\btheta}-{\btheta}^*}-h_{{\btheta}',{\btheta}'-{\btheta}^*}\|_\infty \nonumber\\
&\leq 16\mixing(2r_\infty\phi_\infty+4\phi_\infty^2\norm{{\btheta}})\norm{{\btheta}-{\btheta}'}\\
&\quad +64\mixing\phi_\infty^2\norm{{\btheta}^\prime-{\btheta}^*}\norm{{\btheta}-{\btheta}^\prime},
\end{align*}
and
\[
\|u^*_{\btheta}-u^*_{{\btheta}'}\|_\infty \le 64\mixing\phi_\infty^2\|{\btheta}-{\btheta}'\|\,.
\]
\end{enumerate}
\end{lemma}

\begin{proof}
\emph{(i) Growth bounds.}
By the feature and reward bounds in Lemma~\ref{lem:bounded},
for $z=(s,s^\prime)$ we have
\begin{align*}
\|\bg({\btheta},z)\|
&= \bigl|r(s,s^\prime) + \gamma \bphi(s^\prime)^\top{\btheta} - \bphi(s)^\top{\btheta}\bigr|
\cdot\|\bphi(s)\|\\
&\le \bigl(r_\infty + (\gamma+1)\phi_\infty\|{\btheta}\|\bigr)\,\phi_\infty
\le \bigl(r_\infty + 2\phi_\infty\|{\btheta}\|\bigr)\,\phi_\infty\,.
\end{align*}
Taking expectations under $\pi_Z$ yields the same type of bound for
$\|\bar \bg({\btheta})\|$, hence
\[
\|\bg({\btheta},z)-\bar \bg({\btheta})\|
\le 2(r_\infty + 2\phi_\infty\|{\btheta}\|)\,\phi_\infty
= 2r_\infty\phi_\infty + 4\phi_\infty^2\|{\btheta}\|\,.
\]
Therefore,
\[
|h_{{\btheta},{\btheta}-{\btheta}^*}(z)|
= \bigl|\langle \bg({\btheta},z)-\bar \bg({\btheta}),{\btheta}-{\btheta}^*\rangle\bigr|
\le (2r_\infty\phi_\infty + 4\phi_\infty^2\|{\btheta}\|)\,\|{\btheta}-{\btheta}^*\|\,.
\]
For $h$, recall that
\[
\varepsilon_z
=r(s,s^\prime)+\gamma\bphi(s^\prime)^\top{\btheta}^* -\bphi(s)^\top{\btheta}^*\,.
\]
Then,
\[
\norm{h(z)}
= \norm{\varepsilon_z\bphi(s)}
\le r_\infty\phi_\infty + 2\phi_\infty^2\|{\btheta}^*\|\,.
\]
For $h_{\btheta}$, recall that
\[
\bA_z=\bphi(s)\bigl(\bphi(s)-\gamma\bphi(s^\prime)\bigr)^\top\,.
\]
Then,
\[
\|\bA_z\|_{\mathrm{op}}
\le \|\bphi(s)\|\bigl(\|\bphi(s)\|+\gamma\|\bphi(s^\prime)\|\bigr)
\le (1+\gamma)\phi_\infty^2
\le 2\phi_\infty^2\,.
\]
Consequently, $\|\bA\|_{\mathrm{op}}\le 2\phi_\infty^2$ and
\[
\|\bdelta_z\|_{\mathrm{op}}
\le
\|\bA\|_{\mathrm{op}}+\|\bA_z\|_{\mathrm{op}}
\le
4\phi_\infty^2\,.
\]
Thus,
\[
\|h_{\btheta}(z)\|
=
\|\bdelta_z({\btheta}-{\btheta}^*)\|
\le
4\phi_\infty^2\norm{{\btheta}-{\btheta}^*}\,.
\]

\medskip\noindent
\emph{(ii) Poisson solution sup-norm bounds.}
It follows from Lemma~\ref{lem:poisson-solution} applied to the real-valued function
$h_{{\btheta},{\btheta}-{\btheta}^*}$ and the $\mathbb{R}^d$ functions $h$ and $h_{\btheta}$.

\medskip\noindent
\emph{(iii) Lipschitz continuity in ${\btheta}$.}
\emph{Scalar case $h_{{\btheta},{\btheta}-{\btheta}^*}$.}
For each $z=(s,s^\prime)$ and ${\btheta}\in\mathbb{R}^d$ recall that
\[
\bA_z:=\bphi(s)\bigl(\bphi(s)-\gamma\bphi(s^\prime)\bigr)^\top,
\qquad
\bb_z:=r(s,s^\prime)\bphi(s),
\]
and
\[
\bA=\mathbb{E}_{\pi_Z}[\bA_Z],
\qquad
\bb=\mathbb{E}_{\pi_Z}[\bb_Z],
\]
so that $\bar \bg({\btheta})=-\bA{\btheta}+\bb$ and $\bg({\btheta},z)=-\bA_z{\btheta}+\bb_z$.
Also, recall that
\[
\bdelta_z=\bA-\bA_z,\text{ and let}
\qquad
\bzeta_z:=\bb_z-\bb\,.
\]
Then
\[
\bg({\btheta},z)-\bar \bg({\btheta})
=
\bdelta_z{\btheta}+\bzeta_z,
\]
and hence
\[
h_{{\btheta},{\btheta}-{\btheta}^*}(z)
=
\bigl\langle\bdelta_z{\btheta}+\bzeta_z,\,{\btheta}-{\btheta}^*\bigr\rangle
\,.
\]
As above, $\|\bdelta_z\|_{\mathrm{op}}\le 4\phi_\infty^2$, and $\|\bb_z\|\le r_\infty\phi_\infty$
implies $\|\bzeta_z\|\le 2r_\infty\phi_\infty$.
We have
\begin{align*}
|h_{{\btheta},{\btheta}-{\btheta}^*}(z)- h_{{\btheta}',{\btheta}'-{\btheta}^*}(z)|
&=|\bigl\langle\bdelta_z{\btheta}+\bzeta_z,\,{\btheta}-{\btheta}^*\bigr\rangle
-\bigl\langle\bdelta_z{\btheta}'+\bzeta_z,\,{\btheta}'-{\btheta}^*\bigr\rangle|\\
&=|\bigl\langle\bdelta_z{\btheta}+\bzeta_z,\,{\btheta}-{\btheta}^*-{\btheta}'+{\btheta}^*\bigr\rangle
+\bigl\langle \bdelta_z{\btheta}+\bzeta_z-(\bdelta_z{\btheta}'+\bzeta_z),\,{\btheta}'-{\btheta}^*\bigr\rangle|\\
&\leq \norm{\bdelta_z{\btheta}+\bzeta_z}\norm{{\btheta}-{\btheta}'}
+\norm{\bdelta_z({\btheta}-{\btheta}^\prime)}\norm{{\btheta}'-{\btheta}^*}\\
&\leq \left(4\phi_\infty^2\norm{{\btheta}}+2r_\infty\phi_\infty\right)\norm{{\btheta}-{\btheta}'}
+4\phi_\infty^2\norm{{\btheta}^\prime-{\btheta}^*}\norm{{\btheta}-{\btheta}^\prime}\,.
\end{align*}
The second inequality follows from
applying Lemma~\ref{lem:poisson-solution} and linearity of the Poisson equation:
\[
\|u^*_{{\btheta},{\btheta}-{\btheta}^*}-u^*_{{\btheta}',{\btheta}'-{\btheta}^*}\|_\infty
\le
16\mixing\|h_{{\btheta},{\btheta}-{\btheta}^*}-h_{{\btheta}',{\btheta}'-{\btheta}^*}\|_\infty\,.
\]
\medskip\noindent
\emph{Vector case $h_{\btheta}$.}
For $h_{\btheta}$,
\[
h_{\btheta}(z)-h_{{\btheta}'}(z)
=
\bdelta_z({\btheta}-{\btheta}'),
\]
so
\[
\|h_{\btheta}(z)-h_{{\btheta}'}(z)\|
\le
\|\bdelta_z\|_{\mathrm{op}}\|{\btheta}-{\btheta}'\|
\le
4\phi_\infty^2\|{\btheta}-{\btheta}'\|\,.
\]
Taking the supremum over $z$ gives
\[
\|h_{\btheta}-h_{{\btheta}'}\|_\infty
\le
4\phi_\infty^2\|{\btheta}-{\btheta}'\|\,.
\]
Thus, by Lemma~\ref{lem:poisson-solution} and linearity of the Poisson equation,
\[
\|u^*_{\btheta}-u^*_{{\btheta}'}\|_\infty
\le
64\mixing\phi_\infty^2\|{\btheta}-{\btheta}'\|\,.
\]
\end{proof}

\subsection{Poisson Martingale decomposition of additive Markov noise}
\phantomsection\label{sec:poisson_repr}
Recall $\mathcal{F}_t:=\sigma(Z_0,\dots,Z_t)$.
Let $u_{t-1}$ be the infinite series solution of the Poisson equation with forcing function $h_{t-1}$ defined as one of $h_{\btheta_{t-1},\btheta_{t-1}-\btheta^*}$,
$h$, or $h_{\btheta_{t-1}}$, that is
\[u_{t-1}-P_Z u_{t-1}
=h_{t-1}\,.
\]
Then, we have
\begin{align*}
      \sum_{t=1}^T \eta_{t-1} h_{t-1}(Z_{t-1}) &= \sum_{t=1}^T \eta_{t-1} \bigl(u_{t-1}(Z_{t-1})-\left(P_Z u_{t-1}\right)(Z_{t-1})\bigr)\\
      &= \sum_{t=1}^T \eta_{t-1} \bigl(u_{t-1}(Z_{t-1})-u_{t-1}(Z_{t})+u_{t-1}(Z_{t})-\left(P_Z u_{t-1}\right)(Z_{t-1})\bigr)\\
      &= \sum_{t=1}^T \eta_{t-1} \bigl(u_{t-1}(Z_{t-1})-u_{t-1}(Z_{t})\bigr)+\sum_{t=1}^T \eta_{t-1} \Delta M_t,
\end{align*}
where $\Delta M_t \coloneqq u_{t-1}(Z_{t})-\left(P_Z u_{t-1}\right)(Z_{t-1})$ is a martingale difference with respect to $\mathcal{F}_t:=\sigma(Z_0,\dots,Z_t)$ since $u_{t-1}\in\mathcal{F}_{t-1}$ and
\[
\E[u_{t-1}(Z_{t})\mid \F_{t-1}] = \E[u_{t-1}(Z_{t})\mid \F_{t-1},u_{t-1}] = \left(P_Z u_{t-1}\right)(Z_{t-1}) \,.
\]
We will use this Poisson decomposition repeatedly throughout the paper.

\noindent Thus, for all $T\ge1$, we obtain
\begin{equation}
M_T:=\sum_{t=1}^T \eta_{t-1}\Delta M_{t},
\quad
R_T
:=
\sum_{t=1}^T \eta_{t-1}\bigl(u_{t-1}(Z_{t-1})-u_{t-1}(Z_t)\bigr),
\end{equation}
and $(M_T)_{T\ge1}$ is a $\R^d$-valued martingale with respect to $(\mathcal{F}_T)_{T\ge1}$. Moreover, $R_T$ enjoys the following representation:
\begin{lemma}[Summation by parts]
\label{lem:R-Abel}
\[
R_T=\eta_0u_0(Z_0)-\eta_Tu_T(Z_T)
-
\sum_{t=1}^T(\eta_{t-1}-\eta_t)u_{t-1}(Z_t)
-
\sum_{t=1}^T\eta_t\bigl(u_{t-1}(Z_t)-u_t(Z_t)\bigr)\,.
\]
\end{lemma}

\begin{proof}
Let $R_T=\sum_{t=1}^T \eta_{t-1}(u_{t-1}(Z_{t-1})-u_{t-1}(Z_t))$ and write
\[
R_T
=
\underbrace{\sum_{t=1}^T \eta_{t-1}u_{t-1}(Z_{t-1})}_{T_T^{(1)}}
-
\underbrace{\sum_{t=1}^T \eta_{t-1}u_{t-1}(Z_t)}_{T_T^{(2)}}\,.
\]
Reindex the first sum as $T_T^{(1)}=\sum_{j=0}^{T-1}\eta_ju_j(Z_j)$.
For the second sum, add and subtract the terms $\eta_t u_t(Z_t)$ and $\eta_t u_{t-1}(Z_t)$:
\[
\eta_{t-1}u_{t-1}(Z_t)
=
\eta_t u_t(Z_t)
+
(\eta_{t-1}-\eta_t)u_{t-1}(Z_t)
+
\eta_t(u_{t-1}(Z_t)-u_t(Z_t))\,.
\]
Summing over $t=1,\dots,T$ yields
\[
T_T^{(2)}
=
\sum_{t=1}^T\eta_tu_t(Z_t)
+
\sum_{t=1}^T(\eta_{t-1}-\eta_t)u_{t-1}(Z_t)
+
\sum_{t=1}^T\eta_t\bigl(u_{t-1}(Z_t)-u_t(Z_t)\bigr)\,.
\]
Hence,
\[
R_T
=
\eta_0u_0(Z_0)-\eta_Tu_T(Z_T)
-
\sum_{t=1}^T(\eta_{t-1}-\eta_t)u_{t-1}(Z_t)
-
\sum_{t=1}^T\eta_t\bigl(u_{t-1}(Z_t)-u_t(Z_t)\bigr)\,.
\]
\end{proof}
\section{Bounded Iterates: Proof of Theorem~\ref{thm:bounded-iterates-main}}
\label{app:proof-bounded-iterates-main}

\subsection{Markov noise and localization}
Let $\be_t:={\btheta}_t-{\btheta}^*$ denote the error at time $t$. From the TD update,
expanding $\be_t$ gives
\begin{align*}
\|\be_t\|^2&=\|{\btheta}_t-{\btheta}^*\|^2
=\|{\btheta}_{t-1}+\eta_{t-1}\bg({\btheta}_{t-1},Z_{t-1})-{\btheta}^*\|^2 \\
&= \|{\btheta}_{t-1}-{\btheta}^*\|^2
 + \eta^2_{t-1}\|\bg({\btheta}_{t-1},Z_{t-1})\|^2
 + 2\eta_{t-1}\langle \bg({\btheta}_{t-1},Z_{t-1}),{\btheta}_{t-1}-{\btheta}^*\rangle\\
 &= \|{\btheta}_{t-1}-{\btheta}^*\|^2
 + \eta^2_{t-1}\|\bg({\btheta}_{t-1},Z_{t-1})\|^2
 + 2\eta_{t-1}\langle \bg({\btheta}_{t-1},Z_{t-1})-\bar \bg({\btheta}_{t-1}),{\btheta}_{t-1}-{\btheta}^*\rangle\\
  &\quad + 2\eta_{t-1}\langle \bar \bg({\btheta}_{t-1}),{\btheta}_{t-1}-{\btheta}^*\rangle
 \,.
\end{align*}
Summing from $t=1$ to $T$ yields
\begin{equation}
\label{eq:main_decom}
\|\be_T\|^2
=\|\be_0\|^2
+ \sum_{t=1}^T \eta^2_{t-1} \|\bg({\btheta}_{t-1},Z_{t-1})\|^2
+ 2\sum_{t=1}^T \eta_{t-1}\langle\bar \bg({\btheta}_{t-1}),{\btheta}_{t-1}-{\btheta}^*\rangle
+ B_T,
\end{equation}
where the Markov bias term is
\begin{equation}
\label{eq:B-def}
B_T
:=
2\sum_{t=1}^T \eta_{t-1}\,
\langle \bg({\btheta}_{t-1},Z_{t-1})-\bar \bg({\btheta}_{t-1}),\,{\btheta}_{t-1}-{\btheta}^*\rangle\,.
\end{equation}
By \eqref{eq:dirichlet-monotone} and the nonnegativity of
$f(\btheta)-f(\btheta^*)$,
\[
\langle\bar \bg({\btheta}_{t-1}),{\btheta}_{t-1}-{\btheta}^*\rangle
=
-\bigl(f({\btheta}_{t-1})-f({\btheta}^*)\bigr)
\le 0.
\]
Thus, we have
\[
\|\be_T\|^2
\leq \|\be_0\|^2
+ \sum_{t=1}^T \eta^2_{t-1} \|\bg({\btheta}_{t-1},Z_{t-1})\|^2
+ B_T\,.
\]
To localize the process, we define
\[
T_{\mathrm{exit}}
:=
\inf\{t\ge 0:\,\|{\btheta}_t\|> \rho \rbase\},
\]
where $T_{\mathrm{exit}}$ is a stopping time with respect to $(\mathcal F_t)$ because ${\btheta}_t$ is $\mathcal F_t$-measurable.
The number $\rho > 2$ will be specified later. Define
\[
\rbase
\coloneqq \max\left\{\norm{{\btheta}_0-{\btheta}^*},\norm{{\btheta}^*},\frac{\rinf}{\phiinf}\right\},
\quad\rmax\coloneqq \rho\rbase,
\]
and the stopped iterates and stepsizes by
\[
\tilde{\btheta}_t:={\btheta}_{t\wedge T_{\mathrm{exit}}},
\qquad
\tilde\eta_t:=
\begin{cases}
\eta_t,& t<T_{\mathrm{exit}},\\
0,& t\ge T_{\mathrm{exit}}\,.
\end{cases}
\]
We define the stopped Markov bias term
\[
\tilde B_T
:=
2\sum_{t=1}^T \tilde\eta_{t-1}
\langle \bg(\tilde {\btheta}_{t-1},Z_{t-1})-\bar \bg(\tilde {\btheta}_{t-1}),\,\tilde {\btheta}_{t-1}-{\btheta}^*\rangle\,.
\]
In particular, on the event $\{T_{\mathrm{exit}}=\infty\}$ we have $\tilde{\btheta}_t={\btheta}_t$, $\tilde\eta_t=\eta_t$, and $\tilde B_t=B_t$ for all $t$.

\subsection{Poisson representation and bounds for the localized Markov noise}

Let $u_{t-1}$ be the infinite series solution of the Poisson equation with forcing function $h_{t-1}:=h_{\tilde {\btheta}_{t-1},\tilde{\btheta}_{t-1}-{\btheta}^*}\in\F_{t-2}$, that is,
\[u_{t-1}-P_Z u_{t-1}
=h_{t-1}\,.
\]
Then, from Section~\ref{sec:poisson_repr}, we have
\begin{align*}
      \sum_{t=1}^T\tilde  \eta_{t-1} h_{t-1}(Z_{t-1}) &= \sum_{t=1}^T\tilde  \eta_{t-1} \bigl(u_{t-1}(Z_{t-1})-\left(P_Z u_{t-1}\right)(Z_{t-1})\bigr)\\
      &= \sum_{t=1}^T\tilde  \eta_{t-1} \bigl(u_{t-1}(Z_{t-1})-u_{t-1}(Z_{t})+u_{t-1}(Z_{t})-\left(P_Z u_{t-1}\right)(Z_{t-1})\bigr)\\
      &= \sum_{t=1}^T\tilde  \eta_{t-1} \bigl(u_{t-1}(Z_{t-1})-u_{t-1}(Z_{t})\bigr)+\sum_{t=1}^T\tilde  \eta_{t-1} \Delta M_t,
\end{align*}
where $\Delta M_t \coloneqq u_{t-1}(Z_{t})-\left(P_Z u_{t-1}\right)(Z_{t-1})$ is a martingale difference with respect to $\mathcal{F}_t:=\sigma(Z_0,\dots,Z_t)$ since $u_{t-1}\in\mathcal{F}_{t-1}$ and
\[
\E[u_{t-1}(Z_{t})\mid \F_{t-1}] = \E[u_{t-1}(Z_{t})\mid \F_{t-1},u_{t-1}] = \left(P_Z u_{t-1}\right)(Z_{t-1}) \,.
\]
\noindent Thus, for all $T\ge1$, we obtain
\begin{equation}
\label{eq:Btilde-decomp}
\tilde B_T
=
2\tilde M_T+2\tilde R_T,
\qquad
\tilde M_T:=\sum_{t=1}^T \tilde\eta_{t-1}\Delta M_{t},
\quad
\tilde R_T
:=
\sum_{t=1}^T \tilde\eta_{t-1}\bigl(u_{t-1}(Z_{t-1})-u_{t-1}(Z_t)\bigr),
\end{equation}
and $(\tilde M_T)_{T\ge1}$ is a real-valued martingale with respect to $(\mathcal{F}_T)_{T\ge1}$.

Moreover, if $\tilde \eta_{t-1}>0$, we have $\|\tilde {\btheta}_{t-1}\|\le \rmax$ and $\|\tilde {\btheta}_{t-1}-{\btheta}^*\|\le2\rmax$. Thus, Lemma~\ref{lem:generic-poisson-growth} implies that
\begin{align*}
\|u_{t-1}\|_\infty
&\leq
16\mixing\left(2\rinf\phiinf+4\phiinf^2\norm{\tilde {\btheta}_{t-1}}\right)\left(\norm{\tilde {\btheta}_{t-1}-{\btheta}^*}\right)
\le
16\mixing\phi_\infty^2(4\rmaxs+ 8\rmaxs)\\
&\leq 192\mixing \phi_\infty^2\rmaxs\,.
\end{align*}
Also,
\[
\norm{\Delta M_t}_{\infty}
\le
\norm{u_{t-1}}_{\infty}+\norm{P_Zu_{t-1}}_{\infty}
\le
2\|u_{t-1}\|_\infty
\le
384\mixing\phi_\infty^2 \rmaxs\,.
\]
Since $\tilde\eta_{t-1}=0$ for all $t> \exit$, the increments $\tilde\eta_{t-1}\Delta M_t$ satisfy
\begin{equation}
\label{eq:deltaM-local-bounds}
\norm{\tilde\eta_{t-1}\Delta M_t}_\infty
\le
\tilde{\eta}_{t-1}384\mixing\phiinfs \rmaxs\,.
\end{equation}
We now use Pinelis' inequality in Lemma~\ref{lem:pinelis}. Applying it to the increments
\[
X_t:=\tilde\eta_{t-1}\Delta M_t,\qquad t\ge1,
\]
with \eqref{eq:deltaM-local-bounds}, we obtain the following anytime high-probability guarantee.

\begin{lemma}[Pinelis' inequality for the stopped martingale]
\label{lem:pinelis-anytime-R}
Let $\tilde M_T:=\sum_{t=1}^T X_t$ with $X_t:=\tilde\eta_{t-1}\Delta M_t$ as above.
Then, for all $\delta\in(0,1)$,
\[
\Pr\left\{
\sup_{T\ge 1} \ |\tilde M_T|
\;\le\;
384 \mixing \phiinfs \rmaxs \sqrt{\sum_{t=0}^\infty {\eta}_t^2} \sqrt{2\log\frac{2}{\delta}}
\right\}
\;\ge\;
1-\delta\,.
\]
\end{lemma}

\begin{proof}
Setting $r\geq Db_*\sqrt{2\log\frac 2 \delta}$ in Lemma~\ref{lem:pinelis}, we have $2\exp\left(-\frac{r^2}{2D^2b^2_*}\right)\leq \delta$.
Since $\mathbb{R}^d$ with $\ell_2$ is a $(2,1)$-smooth Banach space, choosing $b_*^2=\sum_{t=0}^\infty {\eta}_t^2 (384\mixing\phiinfs\rmaxs)^2$ finishes the proof.
\end{proof}

\subsection{Bound on the remainder term}
We now bound the deterministic remainder $\tilde R_T$ defined in
\eqref{eq:Btilde-decomp} using Lemma~\ref{lem:R-Abel} and the Lipschitz
continuity of $u_{t-1}$.

\begin{lemma}[Bound on the localized remainder]
\label{lem:remainder}
In the setting of Section~\ref{sec:assumptions}, with the constants from
Lemmas~\ref{lem:bounded} and~\ref{lem:mixing}, for any $T\ge1$,
\[
|\tilde R_T| \le
576\eta_0\mixing\phiinfs \rmaxs
+
672\mixing\phiinf^4\rmaxs\sum_{t=1}^T\eta_{t-1}^2\,.
\]
\end{lemma}

\begin{proof}
Using Lemma~\ref{lem:R-Abel}, we have
\[
\tilde{R}_T=\tilde\eta_0u_0(Z_0)-\tilde\eta_Tu_T(Z_T)
-
\sum_{t=1}^T(\tilde\eta_{t-1}-\tilde\eta_t)u_{t-1}(Z_t)
-
\sum_{t=1}^T\tilde\eta_t\bigl(u_{t-1}(Z_t)-u_t(Z_t)\bigr)\,.
\]
If $\tilde \eta_{t-1}>0$, we have $\|\tilde {\btheta}_{t-1}\|\le \rmax$, hence
$\|\tilde {\btheta}_{t-1}-{\btheta}^*\|\le 2\rmax$.  Lemma~\ref{lem:generic-poisson-growth}
implies
\[
\|u_{t-1}\|_\infty
\le 192\mixing\phiinfs \rmaxs\,.
\]
Therefore
\[
|\tilde\eta_0u_0(Z_0)|+|\tilde\eta_Tu_T(Z_T)|
\le
384\eta_0\mixing\phiinfs \rmaxs\,.
\]
Since $(\eta_{t-1})$ is non-increasing, we have
$\sum_{t=1}^T|\tilde\eta_{t-1}-\tilde\eta_t|=\sum_{t=1}^T(\tilde\eta_{t-1}-\tilde\eta_t)\le\eta_0$, so
\[
\Bigl|
\sum_{t=1}^T(\tilde\eta_{t-1}-\tilde\eta_t)u_{t-1}(Z_t)
\Bigr|
\le
192\eta_0\mixing\phiinfs \rmaxs\,.
\]
For the $u_{t}-u_{t-1}$ term, when $t$ satisfies $\tilde \eta_{t}>0$, we expand the TD recursion
$\tilde{\btheta}_t-\tilde{\btheta}_{t-1}=\tilde\eta_{t-1}\bg(\tilde{\btheta}_{t-1},Z_{t-1})$ and use Lemma~\ref{lem:generic-poisson-growth}(iii) to conclude
\begin{align*}
\norm{u_t-u_{t-1}}_\infty
&\leq 16\mixing\left(2r_\infty\phi_\infty+4\phi_\infty^2\norm{\tilde {\btheta}_t}\right)\norm{\tilde{\btheta}_t-\tilde{\btheta}_{t-1}}
+64\mixing\phi_\infty^2\norm{\tilde{\btheta}_{t-1}-{\btheta}^*}\norm{\tilde {\btheta}_t-\tilde{\btheta}_{t-1}}\\
&\leq \tilde \eta_{t-1} 16\mixing(2r_\infty\phi_\infty+4\phi_\infty^2\rmax)(r_\infty\phiinf+2\phiinfs\rmax)\\
&\quad +64\tilde \eta_{t-1}\mixing\phi_\infty^2 2\rmax (r_\infty\phiinf+2\phiinfs\rmax)\\
&\leq \tilde \eta_{t-1}\phiinf^4\rmaxs(16\times18+64\times2\times3)\mixing\\
&\leq 672\tilde\eta_{t-1}\mixing\phiinf^4\rmaxs\,.
\end{align*}
Thus,
\begin{align*}
\left|\sum_{t=1}^T\tilde\eta_t\bigl(u_{t-1}(Z_t)-u_t(Z_t)\bigr)\right|
&\le
\sum_{t=1}^T\tilde\eta_t\|u_t-u_{t-1}\|_\infty
\le
672\mixing\phiinf^4\rmaxs\sum_{t=1}^T\tilde \eta_t\tilde\eta_{t-1}\\
&\le
672\mixing\phiinf^4\rmaxs\sum_{t=1}^T\eta_{t-1}^2\,.
\end{align*}
Combining the three bounds, we have
\begin{align*}
|\tilde R_T|
& \le
576\eta_0\mixing\phiinfs \rmaxs
+
672\mixing\phiinf^4\rmaxs\sum_{t=1}^T\eta_{t-1}^2\,.
\end{align*}
\end{proof}

\subsection{High-probability bound on the localized Markov noise}

Combining Lemmas~\ref{lem:pinelis-anytime-R} and~\ref{lem:remainder}
with the decomposition \eqref{eq:Btilde-decomp} yields the following
high-probability control of $\tilde B_T$.

\begin{lemma}[Localized high-probability bound for $B_T$]
\label{lem:Btilde-hp-any}
In the setting of Section~\ref{sec:assumptions}, with the constants from
Lemmas~\ref{lem:bounded} and~\ref{lem:mixing}, for any $\delta\in(0,1)$, we have
\[
\Pr\left\{
\frac12 \sup_{T\ge 1} \ |\tilde B_T|
\;\le\;
384\mixing \phiinfs \rmaxs \sqrt{\sum_{t=0}^\infty {\eta}_t^2} \sqrt{2\log\frac{2}{\delta}}
+ C
\right\}
\;\ge\;
1-\delta,
\]
where $C=576\eta_0\mixing\phiinfs \rmaxs +672\mixing\phiinf^4\rmaxs\sum_{t=0}^\infty\eta_{t}^2$.
\end{lemma}

\begin{proof}
From \eqref{eq:Btilde-decomp}, $\tilde B_T=2\tilde M_T+2\tilde R_T$. By Lemma~\ref{lem:pinelis-anytime-R}, with probability at least $1-\delta$,
\[
\sup_{T\ge 1} \ |\tilde M_T|
\le
384 \mixing \phiinfs \rmaxs \sqrt{\sum_{t=0}^\infty {\eta}_t^2} \sqrt{2\log\frac{2}{\delta}}\,.
\]
By Lemma~\ref{lem:remainder}, we also have the deterministic bound
\[
\sup_{T\ge 1} \ |\tilde R_T|
\le
576\eta_0\mixing\phiinfs \rmaxs
+
672\mixing\phiinf^4\rmaxs\sum_{t=0}^\infty\eta_{t}^2\,.
\]
Combining these two bounds yields the desired result.
\end{proof}

\subsection{High-probability bounded iterates via bootstrap}

We can now bootstrap on $\rmax$ to prove high-probability bounded
iterates, completing the proof of
Theorem~\ref{thm:bounded-iterates-main}.

\begin{lemma}[Bootstrap inequality for the radius]
\label{lem:bootstrap-ineq}
In the setting of Section~\ref{sec:assumptions}, with the constants from
Lemmas~\ref{lem:bounded} and~\ref{lem:mixing}, for all $\delta\in(0,1)$, with probability at least $1-\delta$, we have
\begin{align*}
\sup_{T\ge 0} \ \|\tilde \be_T\|^2
\;&\le\;
\|\be_0\|^2
+768\mixing \phiinfs \rmaxs \sqrt{\sum_{t=0}^\infty {\eta}_t^2} \sqrt{2\log\frac{2}{\delta}} + 1152\eta_0\mixing\phiinfs \rmaxs
\\ &\quad\; +
1344\mixing\phiinf^4\rmaxs\sum_{t=0}^\infty\eta_{t}^2+ 9\phiinf^4\rmaxs \sum_{t=0}^\infty\eta_t^2,
\end{align*}
where $\tilde \be_T:=\tilde{\btheta}_T-{\btheta}^*$.
\end{lemma}

\begin{proof}
Using inequality~\eqref{eq:main_decom} and dropping the nonpositive $2\sum_{t=1}^T \eta_{t-1}\langle\bar \bg({\btheta}_{t-1}),{\btheta}_{t-1}-{\btheta}^*\rangle$ term by~\eqref{eq:dirichlet-monotone}, we obtain
\[
\|\tilde \be_T\|^2
\le
\|\be_0\|^2
+
\sum_{t=1}^{T}\tilde\eta_{t-1}^2\|\bg(\tilde {\btheta}_{t-1},Z_{t-1})\|^2
+
\tilde B_T\,.
\]
Lemma~\ref{lem:bounded} implies
\[
\|\bg(\tilde {\btheta}_{t-1},Z_{t-1})\|
\le
\rinf\phiinf + 2\phiinfs \rmax\leq 3\phiinfs\rmax\,.
\]
Thus,
\[
\sum_{t=1}^{T}\tilde \eta_{t-1}^2\|\bg(\tilde {\btheta}_{t-1},Z_{t-1})\|^2
\le
 9\phiinf^4\rmaxs \sum_{t=0}^\infty\eta_t^2\,.
\]
Combining this with Lemma~\ref{lem:Btilde-hp-any} finishes the proof.
\end{proof}
We are now ready to give the proof of Theorem~\ref{thm:bounded-iterates-main}.
\begin{proof}
By the triangle inequality,
\begin{align*}
\sup_{T\ge 0} \ \|\tilde{\btheta}_T\|^2
&\le
2\|{\btheta}^*\|^2 + 2\sup_T \ \|\tilde \be_T\|^2\\
&\le
2\|{\btheta}^*\|^2 + 2\|{\btheta}^*-{\btheta}_0\|^2+
1536\mixing \phiinfs \rmaxs \sqrt{\sum_{t=0}^\infty {\eta}_t^2} \sqrt{2\log\frac{2}{\delta}}\\
&\quad  + 2304\eta_0\mixing\phiinfs \rmaxs+
(2688\mixing+18)\phiinf^4\rmaxs\sum_{t=0}^\infty\eta_{t}^2\\
&\leq \rbases\left(2+2+1536\mixing\phiinfs\rho^2\sqrt{\sum_{t=0}^\infty {\eta}_t^2}\sqrt{2\log\frac{2}{\delta}}\right)\\ &\quad + \rbases\left(2304\eta_0\mixing\phiinfs\rho^2+(2688\mixing+18)\phiinf^4\rho^2\sum^\infty_{t=0}\eta^2_t\right)\,.
\end{align*}
Recall that $\eta_{\mathrm{base}} = \frac{1}{c\mixing\phiinfs}$ and $\sqrt{\sum_{t=0}^\infty {\eta}_t^2} = \frac{1}{c\mixing\phiinfs}\sqrt{\sum_{t=0}^{\infty} a_t^2}$. It follows that
\begin{align*}
\sup_{T\ge 0} \ \|\tilde{\btheta}_T\|^2
&\leq \rbases\left(4+1536\frac{\rho^2}{c}\sqrt{\sum_{t=0}^{\infty} a_t^2}\sqrt{2\log\frac{2}{\delta}}+2304\frac{\rho^2}{c}+2706\left(\sum_{t=0}^{\infty} a^2_t\right)\frac{\rho^2}{c^2}\right)\,.
\end{align*}
There exists a sufficiently large $c$ such that
\[
4+1536\frac{\rho^2}{c}\sqrt{\sum_{t=0}^{\infty} a^2_t}\sqrt{2\log\frac{2}{\delta}}
+2304\frac{\rho^2}{c}
+2706\left(\sum_{t=0}^{\infty} a^2_t\right)\frac{\rho^2}{c^2}
\leq \rho^2.
\]
For such $c$, we have
\[
\sup_{T\ge 0} \ \|\tilde{\btheta}_T\|^2
\le
\rho^2\rbases = \rmaxs\,.
\]
On the event $\{\sup_T \ \|\tilde{\btheta}_T\|^2\le \rmaxs\}$ we must
have $\exit=\infty$ by the definition of $\exit$. Thus $\tilde{\btheta}_T={\btheta}_T$
for all $T$, and
\[
\sup_{T\ge 0} \ \|{\btheta}_T\|^2
\le \rmaxs\,.
\]
\end{proof}

\section{Convergence Rates: Proof of Theorem~\ref{thm:PR-main}}
\label{app:proof-PR-main}
Recall that the TD update can be written as
\[
\bg({\btheta}_{t-1},Z_{t-1})=-\bA_{Z_{t-1}}{\btheta}_{t-1}+\bb_{Z_{t-1}}.
\]
Therefore, the centered error $\be_t:={\btheta}_t-{\btheta}^*$ evolves as
\begin{align}
\be_{t}
&=
\be_{t-1} + \eta_{t-1}\bigl(\bb_{Z_{t-1}}-\bA_{Z_{t-1}}{\btheta}_{t-1}\bigr)\notag\\
&=
\be_{t-1} + \eta_{t-1}\bigl(-\bA_{Z_{t-1}}({\btheta}^*+\be_{t-1})+\bb_{Z_{t-1}}\bigr)\notag\\
&=
\be_{t-1} + \eta_{t-1}\bigl(-\bA_{Z_{t-1}}\be_{t-1} + \bxi_{t-1}\bigr)\notag\\
&=
(\bI_d-\eta_{t-1}\bA)\be_{t-1} + \eta_{t-1}\bxi_{t-1}+\eta_{t-1}\bdelta_{t-1} \be_{t-1}, \label{eq:e-recursion}
\end{align}
using $\bA_{Z_{t-1}}=\bA-\bdelta_{t-1}$ and $\bxi_{t-1}=\bb_{Z_{t-1}}-\bA_{Z_{t-1}}{\btheta}^*$.  Rearranging yields
\begin{equation}
\label{eq:A-e-identity}
\bA \be_{t-1}
=
\frac{\be_{t-1}-\be_{t}}{\eta_{t-1}}
+\bxi_{t-1}+\bdelta_{t-1} \be_{t-1}\,.
\end{equation}
Define the following quantities:
\[
S_T
:=
\sum_{t=1}^{T}\eta_{t-1},
\qquad
\bar{\be}_T
:=
\frac{1}{S_T}\sum_{t=1}^{T}\eta_{t-1} \be_{t-1},
\qquad
\bar{\btheta}_T:={\btheta}^*+\bar{\be}_T\,.
\]
Multiplying \eqref{eq:A-e-identity} by $\eta_{t-1}$ and summing from $t=1$ to $T$ yields
\[
\sum_{t=1}^{T}\eta_{t-1}\bA \be_{t-1}
=
\sum_{t=1}^{T}(\be_{t-1}-\be_{t})
+
\sum_{t=1}^{T}\eta_{t-1}\bxi_{t-1}
+
\sum_{t=1}^{T}\eta_{t-1}\bdelta_{t-1} \be_{t-1}\,.
\]
Dividing by $S_T$ and recalling the definition of $\bar{\be}_T$ yields
\begin{equation}
\label{eq:PR-decomp-eta}
\bA\bar{\be}_T
=
I_1 + I_2 + I_3,
\end{equation}
where
\[
I_1:=\frac{1}{S_T}\sum_{t=1}^{T}(\be_{t-1}-\be_{t}),
\quad
I_2:=\frac{1}{S_T}\sum_{t=1}^{T}\eta_{t-1}\bxi_{t-1},
\quad
I_3:=\frac{1}{S_T}\sum_{t=1}^{T}\eta_{t-1}\bdelta_{t-1} \be_{t-1}\,.
\]
Notice that, by Lemma~\ref{lem:pos-stable}, $(\bA+\bA^\top)/2\succeq\omega \bI_d$, so
for any $\bx\in\mathbb{R}^d$,
\[
\langle \bA \bx,\bx\rangle
=
\bx^\top\frac{\bA+\bA^\top}{2}\bx
\ge
\omega\|\bx\|^2\,.
\]
Combining this with Cauchy--Schwarz,
\[
\omega\|\bx\|^2
\le
\langle \bA \bx,\bx\rangle
\le
\|\bA \bx\|\|\bx\|,
\]
we obtain, for $\bx\neq0$,
\[
\|\bx\|
\le
\frac{1}{\omega}\,\|\bA \bx\|.
\]
In particular, for $\bx=\bar{\be}_T$,
\begin{equation}
\label{eq:fast_rate_decomp}
      f(\bar{\btheta}_T)-f({\btheta}^*)
=
\langle \bA \bar{\be}_T,\bar{\be}_T\rangle
\le
\|\;\bA \bar{\be}_T\|_2\,\|\bar{\be}_T\|
\le
\frac{1}{\omega}\,\|\bA \bar{\be}_T\|^2\,.
\end{equation}
We now bound $\norm{I_1},\norm{I_2}$, and $\norm{I_3}$ in turn to control $f(\bar{\btheta}_T)-f({\btheta}^*)$.

\subsection{Bound on the telescoping term $I_1$}

The first term is purely deterministic.

\begin{lemma}[Bound on $I_1$]
\label{lem:I1-bound}
On the bounded-iterates event
$\mathcal{E}_R:=\{\sup_t \ \|{\btheta}_t\|\le \rmax\}$ we have,
\[
\|I_1\|_2
\le
\frac{2\rmax}{\sum_{t=1}^{T}\eta_{t-1}}
\]
for all $T\geq 1$.
\end{lemma}

\begin{proof}
By telescoping,
\[
\sum_{t=1}^{T}(\be_{t-1}-\be_{t}) = \be_0-\be_T,
\]
so
\[
I_1
=
\frac{\be_0-\be_T}{S_T}.
\]
On $\mathcal{E}_R$,
$\|\be_0-\be_T\|=\|{\btheta}_0-{\btheta}^*-{\btheta}_{T}+{\btheta}^*\|\le 2 \rmax$. Therefore,  $\|I_1\|
\le
\frac{2\rmax}{\sum_{t=1}^{T}\eta_{t-1}}$.
\end{proof}

\subsection{Bound on $I_2$ via vector-valued Martingale Inequality}

\begin{lemma}[Bounds for $I_2$]
\label{lem:I2-bound}
In the setting of Section~\ref{sec:assumptions}, with the constants from
Lemmas~\ref{lem:bounded} and~\ref{lem:mixing}, let
\[
I_2
:=
\frac{1}{S_T}\sum_{t=1}^{T}\eta_{t-1}\bxi_{t-1},
\]
where $\bxi_{t-1}:=\bxi(Z_{t-1})$ and the stepsize sequence
$(\eta_{t-1})$ is non-increasing. Then, with probability at least $1-\delta$, for all $T\geq 1$, we have
\[
\|I_2\|\leq
\frac{1}{\sum_{t=1}^{T}\eta_{t-1}}\left(2\sqrt{\sum_{t=0}^{\infty} \eta_t^2}\sqrt{2\log\frac 2 \delta}+3\eta_0\right)16\mixing(\rinf\phiinf+2\phiinfs\norm{{\btheta}^*})\,.
\]
\end{lemma}

\begin{proof}
By Lemma~\ref{lem:poisson-solution}, we have
\[
u-P_Zu=\bxi,
\quad \text{ and } \quad
\|u\|_\infty\le 16\mixing (\rinf\phiinf+2\phiinfs\norm{{\btheta}^*})\,.
\]
Then, reasoning as in Section~\ref{sec:poisson_repr}, we obtain
\begin{align*}
\sum_{t=1}^{T}\eta_{t-1}\bxi_{t-1}
&=
\sum_{t=1}^{T}\eta_{t-1} \left(u(Z_{t-1})-\left(P_Z u\right)(Z_{t-1})\right)\\
&=\sum_{t=1}^{T}\eta_{t-1} \left(u(Z_{t-1})-u(Z_{t})+u(Z_{t})-\left(P_Z u\right)(Z_{t-1})\right)\\
&=M_T+R_T,
\end{align*}
where
\begin{align*}
M_T
&:=
\sum_{t=1}^{T}\eta_{t-1}\Delta M_{t},\quad
R_T
:=
\sum_{t=1}^{T}\eta_{t-1}\left(u(Z_{t-1})-u(Z_{t})\right)\,.
\end{align*}
Since $\|u\|_\infty\le 16\mixing(\rinf\phiinf+2\phiinfs\norm{{\btheta}^*})$,  we have
\[
\|\Delta M_{t}\|_\infty
\le
\|u\|_\infty
+
\bigl\|P_Z u\bigr\|_\infty
\le
2\|u\|_\infty
\le
32\mixing (\rinf\phiinf+2\phiinfs\norm{{\btheta}^*})\,.
\]
Using Lemma~\ref{lem:pinelis}, for any $\delta\in(0,1)$, we have with probability at least $1-\delta$,
\[
\sup_{T\geq 1} \ \norm{M_T}
\leq 32\mixing (\rinf\phiinf+2\phiinfs\norm{{\btheta}^*}_2)\sqrt{\sum_{t=0}^{\infty} \eta_t^2}\sqrt{2\log\frac 2 \delta}\,.
\]
For the remainder term $R_T$, we have
\begin{align*}\norm{R_T}
&=\norm{\sum_{t=1}^{T}\eta_{t-1}\bigl(u(Z_{t-1})-u(Z_{t})\bigr)}
\leq 3\eta_0\norm{u}_\infty \leq 48\mixing (\rinf\phiinf+2\phiinfs\norm{{\btheta}^*})\eta_0,
\end{align*}
since $\eta_{t-1}$ is non-increasing.
Thus, with probability at least $1-\delta$,
\[
\sup_{T\geq 1} \ \|M_T+R_T\|
\le
\left(2\sqrt{\sum_{t=0}^{\infty} \eta_t^2}\sqrt{2\log\frac 2 \delta}+3\eta_0\right)16\mixing(\rinf\phiinf+2\phiinfs\norm{{\btheta}^*})\,.
\]
Hence, with probability at least $1-\delta$,
\[
\|I_2\|
=
\frac{\|M_T+R_T\|_2}{S_T}
\le
\frac{1}{\sum_{t=1}^{T}\eta_{t-1}}\left(2\sqrt{\sum_{t=0}^{\infty} \eta_t^2}\sqrt{2\log\frac 2 \delta}+3\eta_0\right)16\mixing(\rinf\phiinf+2\phiinfs\norm{{\btheta}^*})\,.
\]
\end{proof}

\subsection{Bound on $I_3$ via localized vector-valued Martingale Inequality}

Recall the stopping time and stopped process:
\[\exit
:=
\inf\{t\ge 0:\,\|{\btheta}_t\|>\rmax\},
\qquad
\tilde{\btheta}_t:={\btheta}_{t\wedge\exit},\qquad
\tilde \be_t:=\tilde{\btheta}_t-{\btheta}^*,
\]
and the stopped stepsizes
\[
\tilde\eta_t
:=
\begin{cases}
\eta_t,& t<\exit,\\[0.3ex]
0,& t\ge\exit\,.
\end{cases}
\]
For $T\ge 1$, define the
localized multiplicative-noise term
\begin{equation}
\label{eq:I3-tilde-def}
\tilde I_{3,T}
:=
\frac{1}{S_T}\sum_{t=1}^{T}\tilde\eta_{t-1}\bdelta_{t-1} \tilde \be_{t-1}\,.
\end{equation}
On the event $\{\exit> T-1\}$ we have $\tilde\eta_{t-1}=\eta_{t-1}$ and $\tilde{\be}_{t-1}=\be_{t-1}$ for all $1\le t\le T$, so $\tilde I_{3,T}=I_3$.

\begin{lemma}[Localized bounds for $I_3$]
\label{lem:I3-bound}
In the setting of Section~\ref{sec:assumptions}, with the constants from
Lemmas~\ref{lem:bounded} and~\ref{lem:mixing}, consider the stopped process defined above. Fix $\delta\in(0,1)$. We have,
with probability at least $1-\delta$, for any $T\geq 1$,
\begin{align*}
\|\tilde I_{3,T}\|_2
\;&\le\; \frac{256\mixing\phiinfs \rmax\sqrt{\sum_{t=0}^{\infty} \eta_t^2}\sqrt{2\log\frac 2 \delta} + 384\eta_0 \mixing\phiinfs \rmax}{S_T}\\
&\quad \; + \frac{192\mixing\phi_\infty^4\rmax\sum_{t=0}^{\infty}\eta_t^2}{S_T}\,.
\end{align*}
\end{lemma}

\begin{proof}
Recall
\[
\bdelta_z=\bA-\bA_z,
\qquad
h_{\tilde {\btheta}}(z):=\bdelta_z(\tilde {\btheta}-{\btheta}^*),
\]
and $\|\bdelta_z\|_{\mathrm{op}}\le 4\phi_\infty^2$ for all
$z$.  Thus, for each $\btheta$, there exists a Poisson solution
$u_{\btheta}:\mathcal{Z}\to\mathbb{R}^d$ satisfying
\[
u_{\btheta} - P_Zu_{\btheta}
=
h_{\btheta},
\]
and Lemma~\ref{lem:generic-poisson-growth} gives
\[
\|u_{\btheta}\|_\infty
\le
64\mixing\phiinfs\norm{{\btheta}-{\btheta}^*},
\qquad
\|u_{\btheta}-u_{{\btheta}'}\|_\infty
\le
64\mixing\phiinfs\|{\btheta}-{\btheta}'\|\,.
\]
Reasoning as in Section~\ref{sec:poisson_repr} with $h_{\tilde {\btheta}_{t-1}}$ yields, for all $T\ge1$,
\begin{align*}
\sum_{t=1}^{T}\tilde\eta_{t-1}\bdelta_{t-1} \tilde \be_{t-1}
&=\sum_{t=1}^{T}\tilde\eta_{t-1} h_{\tilde {\btheta}_{t-1}}(Z_{t-1}) = \sum_{t=1}^{T}\tilde\eta_{t-1} \left(u_{\tilde {\btheta}_{t-1}}(Z_{t-1})-\left(P_Z u_{\tilde {\btheta}_{t-1}}\right)(Z_{t-1})\right)\\
&= \sum_{t=1}^{T}\tilde\eta_{t-1} \left(u_{\tilde {\btheta}_{t-1}}(Z_{t-1})-u_{\tilde{\btheta}_{t-1}}(Z_{t})+u_{\tilde{\btheta}_{t-1}}(Z_{t})-\left(P_Z u_{\tilde {\btheta}_{t-1}}\right)(Z_{t-1})\right)\,.
\end{align*}
Let $\Delta M_{t}\coloneqq u_{\tilde{\btheta}_{t-1}}(Z_{t})-\left(P_Z u_{\tilde {\btheta}_{t-1}}\right)(Z_{t-1})$. We have $\sum_{t=1}^{T}\tilde\eta_{t-1}\bdelta_{t-1} \tilde \be_{t-1}=\tilde M_T + \tilde R_T$, where
\begin{align*}
\tilde M_T
&:=
\sum_{t=1}^{T}\tilde\eta_{t-1}\Delta M_{t},\\
\tilde R_T
&:=
\sum_{t=1}^{T}\tilde\eta_{t-1}
\bigl(u_{\tilde{\btheta}_{t-1}}(Z_{t-1})-u_{\tilde{\btheta}_{t-1}}(Z_{t})\bigr)\,.
\end{align*}
If $\tilde \eta_{t-1}>0$, then $\norm{\tilde {\btheta}_{t-1}} \leq \rmax$ and we have
\[
\|\tilde \eta_{t-1}\Delta M_{t}\|_2
\le 2\tilde \eta_{t-1}\norm{u_{\tilde {\btheta}_{t-1}}}_{\infty}
\le 256\tilde \eta_{t-1} \mixing\phiinfs \rmax\,.
\]
Using Lemma~\ref{lem:pinelis}, for any $\delta\in(0,1)$, we have with probability at least $1-\delta$,
\[
\sup_{T\geq 1} \ \norm{\tilde M_T}_2
\leq 256\mixing\phiinfs \rmax\sqrt{\sum_{t=0}^{\infty} \eta_t^2}\sqrt{2\log\frac 2 \delta}\,.
\]
For the remainder term $\tilde R_T$, using Lemma~\ref{lem:R-Abel}, we have
\[\tilde{R}_T=\tilde\eta_0u_{\tilde {\btheta}_0}(Z_0)-\tilde\eta_{T}u_{\tilde{\btheta}_T}(Z_T)
-
\sum_{t=1}^T(\tilde\eta_{t-1}-\tilde\eta_{t})u_{\tilde {\btheta}_{t-1}}(Z_t)
-
\sum_{t=1}^T\tilde\eta_{t}\bigl(u_{\tilde {\btheta}_{t-1}}(Z_t)-u_{\tilde{\btheta}_t}(Z_t)\bigr)\,.\]
Since $\eta_0\geq\eta_{T-1}$, we obtain
\[
\|\tilde\eta_0 u_{\tilde{\btheta}_0}(Z_0)\|_2
+
\|\tilde\eta_{T} u_{\tilde {\btheta}_T}(Z_T)\|_2
\le
256\eta_0\mixing \phiinfs\rmax\,.
\]
For the second term, we have
\[
\left\|\sum_{t=1}^{T}(\tilde\eta_{t-1}-\tilde\eta_{t})u_{\tilde{\btheta}_{t-1}}(Z_{t})\right\|_2
\le 128\mixing\phiinfs\rmax\sum_{t=1}^{T}|\tilde\eta_{t-1}-\tilde\eta_{t}|\,.
\]
Since the stopped stepsizes $(\tilde\eta_t)$ are non-increasing and satisfy
$\sum_{t=1}^{T}(\tilde\eta_{t-1}-\tilde\eta_{t})
=\tilde\eta_0-\tilde\eta_{T}\le\eta_0$, we have
\[
\left\| \sum_{t=1}^{T}(\tilde\eta_{t-1}-\tilde\eta_{t})u_{\tilde {\btheta}_{t-1}}(Z_{t})
\right\|_2
\le 128\eta_0\mixing\phiinfs\rmax\,.
\]
For the last term, by the Lipschitz property of $u_{\tilde{\btheta}}$ in part~(iii) of Lemma~\ref{lem:generic-poisson-growth}, we have
\[
\|u_{\tilde{\btheta}_{t-1}}-u_{\tilde{\btheta}_{t}}\|_\infty
\le
64\mixing\phiinfs\|\tilde{\btheta}_{t-1}-\tilde{\btheta}_{t}\|_2\,.
\]
Expanding the TD recursion, we have
\[
\|\tilde{\btheta}_{t}-\tilde{\btheta}_{t-1}\|_2
=
\tilde\eta_{t-1}\|\bg(\tilde{\btheta}_{t-1},Z_{t-1})\|_2
\le
\tilde\eta_{t-1}\bigl(r_\infty+2\phi_\infty\|\tilde{\btheta}_{t-1}\|_2\bigr)\phi_\infty
\le
3\tilde \eta_{t-1}\phi_\infty^2\rmax\,.
\]
Therefore,
\[
\begin{aligned}
\left\| \sum_{t=1}^{T}\tilde\eta_{t}\bigl(u_{\tilde{\btheta}_{t-1}}(Z_{t})-u_{\tilde{\btheta}_{t}}(Z_{t})\bigr)\right\|_2
&\le \sum_{t=1}^{T}\tilde\eta_{t} \|u_{\tilde{\btheta}_{t-1}}-u_{\tilde{\btheta}_{t}}\|_\infty\\
&\le 192\mixing\phi_\infty^4\rmax\sum_{t=1}^{T}\tilde\eta_{t}\tilde\eta_{t-1}\\
&\le 192\mixing\phi_\infty^4\rmax\sum_{t=0}^{\infty}\eta_t^2,
\end{aligned}
\]
since $(\tilde \eta_t)$ is non-increasing.
Combining all the estimates finishes the proof.
\end{proof}

\subsection{Proof of Theorem~\ref{thm:PR-main}}

\begin{proof}
Fix $\delta\in (0,1)$ and choose $(c,\rho)$ with $\rho>2$ satisfying
\[
4+1536\frac{\rho^2}{c}\sqrt{\sum_{t=0}^{\infty} a^2_t}\sqrt{2\log\frac{8}{\delta}}+2304\frac{\rho^2}{c}+2706\left(\sum_{t=0}^{\infty} a^2_t\right)\frac{\rho^2}{c^2}\leq \rho^2\,.
\]
Then, with probability at least $1-\delta/4$, running TD(0) with the stepsize schedule $(\eta_t)$ defined in Theorem~\ref{thm:bounded-iterates-main} guarantees
\[
\sup_{t\geq 0} \ \norm{{\btheta}_t}
\leq \rho \max\left\{\norm{{\btheta}_0-{\btheta}^*},\norm{{\btheta}^*}, \frac{r_\infty}{\phi_\infty}\right\},
\]
where we define the corresponding event $\mathcal{E}_R:=\left\{\sup_{t\geq 0} \ \norm{{\btheta}_t}\leq \rho \max\left\{\norm{{\btheta}_0-{\btheta}^*},\norm{{\btheta}^*}, \frac{r_\infty}{\phi_\infty}\right\}\right\}$.

Next, we apply Lemma~\ref{lem:I2-bound} with confidence parameter $\delta/4$
to obtain an event $\mathcal{E}_2$ such that
\[
\Pr\{\mathcal{E}_2\} \ge 1 - \frac{\delta}{4},
\text{ and on }\mathcal{E}_2:
\|I_2\|
\le
\frac{2}{S_T}\left(2\sqrt{\sum_{t=0}^{\infty} \eta_t^2}\sqrt{2\log\frac 8 \delta}+3\eta_0\right)8\mixing(\rinf\phiinf+2\phiinfs\norm{{\btheta}^*})\,.
\]
Similarly, apply Lemma~\ref{lem:I3-bound} with the same $\delta/4$ to
obtain an event $\mathcal{E}_3$ such that
\[
\Pr\{\mathcal{E}_3\}\ge1-\frac{\delta}{4},
\qquad
\text{and on }\mathcal{E}_3:
\]
\begin{align*}
\|\tilde I_{3,T}\|
&\le \frac{2}{S_T}\left(128\mixing\phiinfs \rmax\sqrt{\sum_{t=0}^{\infty} \eta_t^2}\sqrt{2\log\frac 8 \delta} + 192\eta_0 \mixing\phiinfs \rmax\right)\\
&\quad + \frac{2}{S_T}\left(96\mixing\phi_\infty^4\rmax\sum_{t=0}^{\infty}\eta_t^2\right)\,.
\end{align*}
Moreover, from the proof of Lemma~\ref{lem:bootstrap-ineq}, we can also obtain an event $\mathcal{E}_4$ such that
\[
\Pr\{\mathcal{E}_4\}\ge1-\frac{\delta}{4},
\qquad
\text{and on }\mathcal{E}_4:
\]
\begin{align*}
\sum_{t=1}^T\tilde \eta_{t-1}^2\norm{\bg(\tilde {\btheta}_{t-1},Z_{t-1})}^2 + \tilde B_T
&\le
768 \mixing \phiinfs \rmaxs \sqrt{\sum_{t=0}^\infty {\eta}_t^2} \sqrt{2\log\frac{8}{\delta}} + 1152\eta_0\mixing\phiinfs \rmaxs\\
&\quad + 1344\mixing\phiinf^4\rmaxs\sum_{t=0}^\infty\eta_{t}^2+ 9\phiinf^4\rmaxs \sum_{t=0}^\infty\eta_t^2\,.
\end{align*}

In particular, if we choose $\rmax=\rho \max\left\{\norm{{\btheta}_0-{\btheta}^*},\norm{{\btheta}^*}, \frac{r_\infty}{\phi_\infty}\right\}$, we have $\exit=\infty$ on $\mathcal{E}_R$. Hence, for all $T\geq 1$,
$\tilde\eta_{t-1}=\eta_{t-1}$ and $\tilde \be_{t-1}=\be_{t-1}$ for $1\le t\le T$, so
$\tilde I_{3,T}=I_3$ and $\sum_{t=1}^T\tilde \eta_{t-1}^2\norm{\bg(\tilde {\btheta}_{t-1},Z_{t-1})}^2 + \tilde B_T =\sum_{t=1}^T\eta_{t-1}^2\norm{\bg({\btheta}_{t-1},Z_{t-1})}^2 + B_T $.
Define
\[
\mathcal{E}
:=
\mathcal{E}_R\cap\mathcal{E}_2\cap\mathcal{E}_3\cap\mathcal{E}_4.
\]
By a union bound, we have
\[
\Pr\{\mathcal{E}\}
\ge
1-\left(\frac{\delta}{4}+\frac{\delta}{4}+\frac{\delta}{4}+\frac{\delta}{4}\right)
= 1-\delta\,.
\]
Moreover, on $\mathcal{E}$,
\begin{align*}
\|\bA\bar{\be}_T\|
&=
\|I_1+I_2+I_3\|
\le
\|I_1\|+\|I_2\|+\|I_3\|\\
&\le \frac{2\rmax}{\sum_{t=1}^{T}\eta_{t-1}}+ \frac{2}{\sum_{t=1}^{T}\eta_{t-1}}\left(2\sqrt{\sum_{t=0}^{\infty} \eta_t^2}\sqrt{2\log\frac 8 \delta}+3\eta_0\right)8\mixing(\rinf\phiinf+2\phiinfs\norm{{\btheta}^*})\\
&\quad +  \frac{2}{\sum_{t=1}^{T}\eta_{t-1}}\left(128\mixing\phiinfs \rmax\sqrt{\sum_{t=0}^{\infty} \eta_t^2}\sqrt{2\log\frac 8 \delta} \right.\\
&\qquad\qquad\qquad \qquad\left. + 192\eta_0 \mixing\phiinfs \rmax + 96\mixing\phi_\infty^4\rmax\sum_{t=0}^{\infty}\eta_t^2\right)\,.
\end{align*}
Thus, using \eqref{eq:fast_rate_decomp}, on $\mathcal{E}$, we have
\[
f(\bar{\btheta}_T)-f({\btheta}^*)
\le
\frac{C_{\mathrm{fast}}^2}{\omega(\sum_{t=1}^{T}\eta_{t-1})^2},
\]
where
\begin{align*}
C_{\mathrm{fast}}&\coloneqq 2\rmax + \left(2\sqrt{\sum_{t=0}^{\infty} \eta_t^2}\sqrt{2\log\frac 8 \delta}+3\eta_0\right)48\mixing\phiinfs\rmax\\
&\quad +256\mixing\phiinfs \rmax \sqrt{\sum_{t=0}^{\infty} \eta_t^2}\sqrt{2\log\frac 8 \delta} + 384\eta_0 \mixing\phiinfs \rmax \\
&\quad + 192\mixing\phi_\infty^4\rmax \sum_{t=0}^{\infty}\eta_t^2\,.
\end{align*}
On the other hand, using \eqref{eq:main_decom}, we also have
\[
\|\be_T\|^2
\;=\;
\|\be_0\|^2
+
\sum_{t=1}^T \eta^2_{t-1}\|\bg({\btheta}_{t-1},Z_{t-1})\|^2+2\sum_{t=1}^T\eta_{t-1}\Ip{\bar \bg({\btheta}_{t-1})}{{\btheta}_{t-1}-{\btheta}^*}
+ B_T\,.
\]
Rearranging terms yields
\begin{align*}
2\sum_{t=1}^T\eta_{t-1}\Ip{\bar \bg({\btheta}_{t-1})}{{\btheta}^*-{\btheta}_{t-1}}&= \norm{\be_0}^2-\norm{\be_T}^2 + \sum_{t=1}^T \eta^2_{t-1}\|\bg({\btheta}_{t-1},Z_{t-1})\|^2+B_T\\
&\leq \rmaxs+ 768 \mixing \phiinfs \rmaxs \sqrt{\sum_{t=0}^\infty {\eta}_t^2} \sqrt{2\log\frac{8}{\delta}} \\
&\quad + 1152\eta_0\mixing\phiinfs \rmaxs\\
&\quad + 1344\mixing\phiinf^4\rmaxs\sum_{t=0}^\infty\eta_{t}^2+ 9\phiinf^4\rmaxs \sum_{t=0}^\infty\eta_t^2\,.
\end{align*}
Using the convexity of $f$ and \eqref{eq:dirichlet-monotone}, we also have
\begin{align*}
f(\bar {\btheta}_T)-f({\btheta}^*)&\leq \left(\frac{1}{\sum_{t=1}^{T}\eta_{t-1}}\right) \sum_{t=1}^T \eta_{t-1}\Ip{\bar \bg({\btheta}_{t-1})}{{\btheta}^*-{\btheta}_{t-1}}\\
&\leq \frac{\rmaxs}{\sum_{t=1}^{T}\eta_{t-1}}\left(0.5 + 384 \mixing \phiinfs \sqrt{\sum_{t=0}^\infty {\eta}_t^2} \sqrt{2\log\frac{8}{\delta}} + 576 \eta_0\mixing\phiinfs\right)\\
&\quad +\frac{\rmaxs}{\sum_{t=1}^{T}\eta_{t-1}}\left(672\mixing\phiinf^4\sum_{t=0}^\infty\eta_{t}^2+ 4.5\phiinf^4\sum_{t=0}^\infty\eta_t^2\right)\\
&\leq \frac{C_{\mathrm{robust}}}{\sum_{t=1}^{T}\eta_{t-1}},
\end{align*}
where
\begin{align*}C_{\mathrm{robust}}&\coloneqq \rmaxs \left(0.5 + 384 \mixing \phiinfs \sqrt{\sum_{t=0}^\infty {\eta}_t^2} \sqrt{2\log\left(\frac{8}{\delta}\right)} \right.\\
      &\quad \left.+576 \eta_0\mixing\phiinfs+672\mixing\phiinf^4\sum_{t=0}^\infty\eta_{t}^2+ 4.5\phiinf^4\sum_{t=0}^\infty\eta_t^2\right).
\end{align*}
Combining the two upper bounds for $f(\bar {\btheta}_T)-f({\btheta}^*)$, with probability at least $1-\delta$, we have
\[
f(\bar {\btheta}_T)-f({\btheta}^*)\leq \min\left\{\frac{C_{\mathrm{fast}}^2}{\omega(\sum_{t=1}^{T}\eta_{t-1})^2},\frac{C_{\mathrm{robust}}}{\sum_{t=1}^{T}\eta_{t-1}}\right\}\,.
\]
\end{proof}

\section{Explicit Derivation of Stepsize Schedules}
\label{sec:stepsize}

In this section, we consider specific stepsize schedules $(\eta_t)$ and derive corresponding corollaries of Theorem~\ref{thm:bounded-iterates-main} and Theorem~\ref{thm:PR-main}.

\begin{corollary}[High-probability bounded iterates]
      \label{cor:bounded-iterates-sqrtlog}
In the setting of Theorem~\ref{thm:bounded-iterates-main}, consider the following stepsize schedule:
\[
\eta_t=\frac{1}{c\mixing\phiinfs\sqrt{t+1}\log(t+3)}, \quad \forall t\geq 0,
\]
for some numerical constant $c>0$. Fix any $\delta\in(0,1)$ and let
$\rbase\coloneqq \max \left\{\norm{{\btheta}_0-{\btheta}^*},\norm{{\btheta}^*},r_\infty/\phi_\infty\right\}$, $A_1(\delta)=1536\sqrt{\frac{3}{\log 2}}\sqrt{2\log\frac{2}{\delta}}+2304$, and $A_2=\frac{8118}{\log 2}$.
Then, provided that $c>c_{\min}(\delta)\coloneqq\frac{A_1(\delta)+\sqrt{A_1^2(\delta)+4A_2}}{2}$, with probability at least $1-\delta$, we have
\[
\sup_{t\geq 0} \ \norm{{\btheta}_t}_2\leq \rho \rbase, \quad \text{where } \rho=\frac{2c}{\sqrt{c^2-A_1(\delta)c - A_2}}\,.
\]
\end{corollary}

\begin{proof}
Fix $\delta\in(0,1)$. We substitute $a_t=\frac{1}{\sqrt{t+1}\log(t+3)},\forall t\geq 0$ into Theorem~\ref{thm:bounded-iterates-main}. Lemma~\ref{lem:stepsize-lower-bound}(i) gives
$\sum_{t=0}^\infty a_t^2\leq \frac{3}{\log 2}$.
Thus, a sufficient requirement for $c$ and $\rho$ translates to
\[4+1536\frac{\rho^2}{c}\sqrt{\frac{3}{\log 2}}\sqrt{2\log\frac{2}{\delta}}+2304\frac{\rho^2}{c}+2706\frac{3}{\log 2}\frac{\rho^2}{c^2}\leq \rho^2\,.\]
We now derive an explicit relationship between $c$ and $\rho$ for the above inequality to hold. Denote $A_1(\delta)=1536\sqrt{\frac{3}{\log 2}}\sqrt{2\log\frac{2}{\delta}}+2304$ and $A_2=2706\frac{3}{\log 2}$. Then, the above inequality is equivalent to
\begin{equation}
      \label{eq:c_rho_condition}
4+\frac{\rho^2}{c}A_1(\delta)+\frac{\rho^2}{c^2}A_2\leq \rho^2\,.
\end{equation}

For any fixed $c>0$, rearranging \eqref{eq:c_rho_condition} yields
\[
c^2-A_1(\delta)c-A_2\geq \frac{4c^2}{\rho^2},
\]
which requires $c^2-A_1(\delta)c - A_2>0$ since the right-hand side is positive. In particular, for any  $c> c_{\min}(\delta)\coloneqq\frac{A_1(\delta)+\sqrt{A_1^2(\delta)+4A_2}}{2}$, choosing
\[
\rho\geq \rho_{\min}(c)\coloneqq \frac{2c}{\sqrt{c^2-A_1(\delta)c - A_2}}
\]
guarantees the condition. This completes the proof.
\end{proof}

\begin{corollary}[High-probability rate for PR averaging]
  \label{cor:PR-main}
Under Assumption~\ref{ass:mrp-ergodic-full-rank}, with
$\phi_\infty,r_\infty,\mixing$ and $\omega$ as in
Lemmas~\ref{lem:bounded}--\ref{lem:pos-stable}, consider the following stepsize schedule:
\[
\eta_t=\frac{1}{c\mixing\phiinfs\sqrt{t+1}\log(t+3)}, \qquad \forall t\geq 0,
\]
for some numerical constant $c>0$. Fix any $\delta\in(0,1)$ and let
$\rbase\coloneqq \max \left\{\norm{{\btheta}_0-{\btheta}^*},\norm{{\btheta}^*},r_\infty/\phi_\infty\right\}$, $\rmax=\rho\rbase$, $A_1(\delta)=1536\sqrt{\frac{3}{\log 2}}\sqrt{2\log\frac{8}{\delta}}+2304$, and $A_2=\frac{8118}{\log 2}$.
Then, provided that $c>c_{\min}(\delta)\coloneqq\frac{A_1(\delta)+\sqrt{A_1^2(\delta)+4A_2}}{2}$, with probability at least $1-\delta$, the following upper bound holds for $f(\bar{\btheta}_T)-f({\btheta}^*)$:
\begin{align*}
\rmaxs\,
\min&\left\{
\underbrace{
\frac{c^2\mixing^2\phi_\infty^4\,\log^2(T+2)}{\omega\,(\sqrt{T+1}-1)^2}
\left(
1+\frac{264}{c}
+\frac{176\sqrt{3}}{c\sqrt{\log 2}}\sqrt{2\log\frac{8}{\delta}}
+\frac{288}{c^2\mixing\log 2}
\right)^2
}_{\text{fast rate}},\right.
\\
&\quad
\left.\underbrace{
\frac{c\mixing\phi_\infty^2\,\log(T+2)}{4(\sqrt{T+1}-1)}
\left(
1+\frac{1152}{c}
+\frac{768\sqrt{3}}{c\sqrt{\log 2}}\sqrt{2\log\frac{8}{\delta}}
+\frac{4059}{c^2\mixing\log 2}
\right)
}_{\text{robust rate}}
\right\},
\end{align*}
where $\rho=\frac{2c}{\sqrt{c^2-A_1(\delta)c - A_2}}$. That is, $f(\bar{\btheta}_T)-f({\btheta}^*)=\min\{\mathcal{O}(\log^2(T)/(\omega T)),\mathcal{O}(\log(T)/\sqrt{T})\}$ up to constants depending on $c$, $\mixing$, $\phi_\infty$, and $\delta$.
\end{corollary}

\begin{proof}
The condition for $(c,\rho)$ can be derived as in Corollary~\ref{cor:bounded-iterates-sqrtlog}. We now focus on simplifying the constants $C_{\mathrm{robust}}$ and $C_{\mathrm{fast}}$. Lemma~\ref{lem:stepsize-lower-bound}(i) gives $\sum_{t=0}^\infty \eta_t^2\leq \frac{3}{c^2\mixing^2\phi_\infty^4\log 2}$. Thus, we have
\begin{align*}
C_{\mathrm{robust}}
&\leq  \rmaxs \left[\frac12
+ 384 \mixing \phiinfs \sqrt{\frac{3}{c^2\mixing^2\phi_\infty^4\log 2}} \sqrt{2\log\frac{8}{\delta}}
+ 576 \eta_0\mixing\phiinfs \right.\\
&\qquad\qquad
\left.+\left(672\mixing+4.5\right)\phi_\infty^4\frac{3}{c^2\mixing^2\phi_\infty^4\log 2}\right]\\
&\leq \rmaxs\left[\frac12
+ \frac{384\sqrt{3}}{c\sqrt{\log 2}}\sqrt{2\log\frac{8}{\delta}}
+ \frac{576}{c}
+ \frac{2029.5}{c^2\mixing\log 2}\right],
\end{align*}
where we used $\eta_0\le 1/(c\mixing\phiinfs)$ and $\mixing\ge 1$, so that
$\frac{2016}{c^2\mixing\log 2}+\frac{13.5}{c^2\mixing^2\log 2}\le \frac{2029.5}{c^2\mixing\log 2}$.
Similarly,
\begin{align*}
C_{\mathrm{fast}}
&\leq \rmax\left[2+528\eta_0\mixing\phiinfs
+352\mixing\phiinfs\sqrt{\frac{3}{c^2\mixing^2\phi_\infty^4\log 2}}\sqrt{2\log\frac 8 \delta}\right.\\
&\quad\quad\quad\quad\left. +192\mixing\phi_\infty^4 \frac{3}{c^2\mixing^2\phi_\infty^4\log 2}\right]\\
&\leq \rmax\left[2+\frac{528}{c}
+ \frac{352\sqrt{3}}{c\sqrt{\log 2}}\sqrt{2\log\frac 8 \delta}
+ \frac{576}{c^2\mixing\log 2}\right]\,.
\end{align*}
Recalling $\rmax=\rho\rbase$, we obtain
\begin{align*}
C_{\mathrm{fast}}
&\le 2\rho\rbase\left(1+\frac{264}{c}
+\frac{176\sqrt{3}}{c\sqrt{\log 2}}\sqrt{2\log\frac 8 \delta}
+\frac{288}{c^2\mixing\log 2}\right),\\
C_{\mathrm{robust}}
&\le \frac{\rho^2\rbases}{2}\left(1+\frac{1152}{c}
+\frac{768\sqrt{3}}{c\sqrt{\log 2}}\sqrt{2\log\frac 8 \delta}
+\frac{4059}{c^2\mixing\log 2}\right).
\end{align*}
Next, let $S_T\coloneqq \sum_{t=1}^{T}\eta_{t-1}=\sum_{s=0}^{T-1}\eta_s$. By Lemma~\ref{lem:stepsize-lower-bound} with $p=1$ and $\kappa=c\mixing\phiinfs$,
$S_T\ge \frac{2(\sqrt{T+1}-1)}{c\mixing\phiinfs\log(T+2)}$.
Therefore,
\[
\frac{1}{S_T}\le \frac{c\mixing\phiinfs\log(T+2)}{2(\sqrt{T+1}-1)},
\qquad
\frac{1}{S_T^2}\le \frac{c^2\mixing^2\phi_\infty^4\log^2(T+2)}{4(\sqrt{T+1}-1)^2}\,.
\]
Substituting the above bounds into Theorem~\ref{thm:PR-main} yields the claimed result.
\end{proof}

\section{Removing the dependence on $\mixing$ in the stepsize}
\label{sec:no_tau_mixing}
The choice in~\eqref{eq:beststepsize} depends on the unknown mixing parameter $\mixing$, which may make the algorithm impractical. This dependence of the stepsize on $\mixing$ can be avoided by running TD(0) with the modified stepsize schedule
\begin{equation}
\label{eq:mixing-free-stepsize}
\eta_t
= \frac{1}{c\,\phiinfs\sqrt{t+1}\log^2(t+3)},
\qquad \forall t\geq 0,
\end{equation}
for some numerical constant $c>0$.  For $T\geq 1$, let
\begin{equation}
\label{eq:full-PR-average}
S_T\coloneqq \sum_{t=1}^{T}\eta_{t-1},
\qquad
\bar{\btheta}_T
\coloneqq
\frac{1}{S_T}\sum_{t=1}^{T}\eta_{t-1}{\btheta}_{t-1}\,.
\end{equation}
\begin{lemma}[Deterministic bound on the finite prefix]
\label{lem:deterministic-prefix}
Assume the feature and reward bounds in Lemma~\ref{lem:bounded}. For any nonnegative stepsizes $(\eta_t)_{t\ge0}$ and any integer $m\ge0$, the TD(0) iterates satisfy
\[
\max_{0\leq s\leq m}\norm{{\btheta}_s}
\leq \left(\norm{{\btheta}_0}+\frac{\rinf}{(1+\gamma)\phiinf}\right) \prod_{t=0}^{m-1} \left(1+(1+\gamma)\phiinfs\eta_t\right) - \frac{\rinf}{(1+\gamma)\phiinf}\,.
\]
In particular, for the stepsize~\eqref{eq:mixing-free-stepsize}, define
\[
B_m
\coloneqq
\left(\norm{{\btheta}_0}+\frac{\rinf}{(1+\gamma)\phiinf}\right)
\prod_{t=0}^{m-1}
\left(
1+
\frac{1+\gamma}{c\sqrt{t+1}\log^2(t+3)}
\right)
-
\frac{\rinf}{(1+\gamma)\phiinf}\,.
\]
Then $\max_{0\leq s\leq m}\norm{{\btheta}_s}\leq B_m$.  Furthermore,
\[
B_m
\leq
\left(\norm{{\btheta}_0}+\frac{\rinf}{(1+\gamma)\phiinf}\right)
\exp\left(
\frac{2(1+\gamma)\sqrt m}{c\log^2 3}
\right)
-
\frac{\rinf}{(1+\gamma)\phiinf}\,.
\]
\end{lemma}

\begin{proof}
The bounds in Lemma~\ref{lem:bounded} imply
\begin{align*}
\norm{\bg({\btheta}_t,Z_t)}
&=
\left|r_t+\gamma\left\langle\bphi(s_{t+1}),{\btheta}_t\right\rangle
-\left\langle\bphi(s_t),{\btheta}_t\right\rangle\right|
\norm{\bphi(s_t)}
\\
&\leq
\rinf\phiinf+(1+\gamma)\phiinfs\norm{{\btheta}_t}\,.
\end{align*}
Therefore,
\[
\norm{{\btheta}_{t+1}}
\leq
(1+(1+\gamma)\phiinfs\eta_t)\norm{{\btheta}_t}
+\rinf\phiinf\eta_t\,.
\]
Let $b\coloneqq \rinf/((1+\gamma)\phiinf)$ and
$u_t\coloneqq \norm{{\btheta}_t}+b$. Then $u_{t+1}\leq(1+(1+\gamma)\phiinfs\eta_t)u_t$.
Iterating this inequality yields, for every $s\leq m$,
\[
u_s
\leq u_0\prod_{t=0}^{s-1}(1+(1+\gamma)\phiinfs\eta_t)
\leq u_0\prod_{t=0}^{m-1}(1+(1+\gamma)\phiinfs\eta_t)\,.
\]
Taking the supremum over $0\leq s\leq m$ and subtracting $b$ proves the first claim.
Substituting~\eqref{eq:mixing-free-stepsize} gives the displayed formula for $B_m$.
Finally, using $1+x\leq e^x$ and
\[
\sum_{t=0}^{m-1}\frac{1}{\sqrt{t+1}\log^2(t+3)}
\leq
\frac{2\sqrt m}{\log^2 3}
\]
gives the closed-form bound.
\end{proof}

\begin{corollary}[Mixing-free stepsize]
\label{cor:mixing-free-via-logt}
Under Assumption~\ref{ass:mrp-ergodic-full-rank}, with
$\phi_\infty,r_\infty,\mixing$ and $\omega$ as in
Lemmas~\ref{lem:bounded}--\ref{lem:pos-stable}, run TD(0) with the stepsize schedule:
\[
\eta_t
=
\frac{1}{c\,\phiinfs\sqrt{t+1}\log^2(t+3)},
\qquad t\ge 0,
\]
for some numerical constant $c>0$, and form the PR average $\bar{\btheta}_T$ as in~\eqref{eq:full-PR-average}.
Let
\[
m^*\coloneqq \min\{m\ge 0:\log(m+3)\ge \mixing\}\,.
\]
For $m\ge0$, define the deterministic prefix bound
\[
B_m
\coloneqq
\left(\norm{{\btheta}_0}+\frac{\rinf}{(1+\gamma)\phiinf}\right)
\prod_{t=0}^{m-1}
\left(
1+
\frac{1+\gamma}{c\sqrt{t+1}\log^2(t+3)}
\right)
-
\frac{\rinf}{(1+\gamma)\phiinf}\,.
\]
Fix any $\delta\in(0,1)$ and set
\[
\rbase
\coloneqq
\max\left\{B_{m^*}+\norm{{\btheta}^*},\ \norm{{\btheta}^*},\ \frac{\rinf}{\phiinf}\right\}\,.
\]
Define
\[
A_1(\delta)
\coloneqq
1536\sqrt{\frac{3}{\log 2}}\sqrt{2\log\frac{8}{\delta}}+2304,
\qquad
A_2
\coloneqq
\frac{8118}{\log 2},
\]
\[
c_{\min}(\delta)
\coloneqq
\frac{A_1(\delta)+\sqrt{A_1^2(\delta)+4A_2}}{2},
\qquad
\rho
\coloneqq
\frac{2c}{\sqrt{c^2-A_1(\delta)c-A_2}},
\qquad
\rmax
\coloneqq
\rho\rbase\,.
\]
Then, provided that $c>c_{\min}(\delta)$, with probability at least $1-\delta$, we have
\[
\sup_{t\ge0} \ \norm{{\btheta}_t}
\le
\rmax,
\]
and for every $T\ge1$, $f(\bar{\btheta}_T)-f({\btheta}^*)$ is upper-bounded by
\begin{align*}
\rmaxs\,
\min&\left\{
\underbrace{
\frac{c^2\phi_\infty^4\log^4(T+2)}
{\omega(\sqrt{T+1}-1)^2}
\left(
1+
\frac{264\mixing}{c\log^2 3}
+
\frac{176\sqrt{3}\mixing}{c\sqrt{\log 2}}\sqrt{2\log\frac{8}{\delta}}
+
\frac{288\mixing}{c^2\log 2}
\right)^2
}_{\text{fast rate}},
\right.
\\
&\quad \left.
\underbrace{
\frac{c\phi_\infty^2\log^2(T+2)}
{4(\sqrt{T+1}-1)}
\left(
1+
\frac{1152\mixing}{c\log^2 3}
+
\frac{768\sqrt{3}\mixing}{c\sqrt{\log 2}}\sqrt{2\log\frac{8}{\delta}}
+
\frac{4032\mixing+27}{c^2\log 2}
\right)
}_{\text{robust rate}}
\right\}\,.
\end{align*}
Equivalently, since $\mixing\ge1$, the horizon dependence is
\[
f(\bar{\btheta}_T)-f({\btheta}^*)
=
\min\left\{
\widetilde{\mathcal O}\!\left(\frac{\rmaxs\mixing^2\phi_\infty^4}{\omega T}\right),
\widetilde{\mathcal O}\!\left(\frac{\rmaxs\mixing\phi_\infty^2}{\sqrt T}\right)
\right\},
\]
where the $\widetilde{\mathcal O}(\cdot)$ notation hides logarithmic factors in $T$ and $1/\delta$, as well as numerical constants depending on $c$.
Moreover, the prefactor depends doubly exponentially on $\mixing$ through its dependence on $\rbase$: using $m^*\le e^{\mixing}$ and Lemma~\ref{lem:deterministic-prefix},
\[
B_{m^*}
\le
\left(\norm{{\btheta}_0}+\frac{\rinf}{(1+\gamma)\phiinf}\right)
\exp\left(
\frac{2(1+\gamma)e^{\mixing/2}}{c\log^2 3}
\right)
-
\frac{\rinf}{(1+\gamma)\phiinf}\,.
\]
Consequently, for fixed $c,\gamma,\phi_\infty,\rinf,{\btheta}_0,{\btheta}^*$, and $\delta$, there is a constant $C_0$ independent of $\mixing$ such that
\[
\rmaxs
\le
C_0
\exp\left(
\frac{4(1+\gamma)e^{\mixing/2}}{c\log^2 3}
\right)
= \exp\left(\mathcal O\bigl(e^{\mixing/2}\bigr)\right)\,.
\]
Thus the total fast and robust prefactors are both of doubly exponential order
$\exp(\mathcal O(e^{\mixing/2}))$ up to polynomial factors in $\mixing$.
\end{corollary}

\begin{proof}
Let
\[
L_\delta\coloneqq\sqrt{2\log\frac{8}{\delta}},
\qquad
H\coloneqq\sum_{t=0}^\infty \eta_t^2\,.
\]
We first prove the bounded-iterates event.  Consider the shifted process
\[
{\btheta}'_k\coloneqq {\btheta}_{m^*+k},
\qquad
Z'_k\coloneqq Z_{m^*+k},
\qquad
\eta'_k\coloneqq\eta_{m^*+k}\,.
\]
The shifted chain has the same transition kernel, hence the same mixing constant $\mixing$.
Then
\[
\eta'_k
=
\frac{1}{c\mixing\phiinfs}a'_k,
\qquad
a'_k
\coloneqq
\frac{\mixing}{\sqrt{m^*+k+1}\log^2(m^*+k+3)}\,.
\]
By the definition of $m^*$, for every $k\ge0$,
\[
0<a'_k
\le
\frac{1}{\sqrt{m^*+k+1}\log(m^*+k+3)},
\qquad
a'_0\le 1,
\]
and $(a'_k)_{k\ge0}$ is non-increasing.  Moreover, Lemma~\ref{lem:stepsize-lower-bound}(i) gives
\begin{align*}
\sum_{k=0}^\infty (a'_k)^2
&=
\mixing^2\sum_{k=0}^\infty
\frac{1}{(m^*+k+1)\log^4(m^*+k+3)}
\\
&\le
\sum_{k=0}^\infty
\frac{1}{(m^*+k+1)\log^2(m^*+k+3)}
\le
\frac{3}{\log 2}\,.
\end{align*}
Conditioning on $\mathcal{F}_{m^*}=\sigma(Z_0,\dots,Z_{m^*})$ and applying Theorem~\ref{thm:bounded-iterates-main} to the shifted recursion with confidence parameter $\delta/4$, we obtain an event $\mathcal{E}_R^+$ such that
\[
\Pr\{\mathcal{E}_R^+\mid\mathcal F_{m^*}\}\ge1-\frac{\delta}{4}
\quad a.s.,
\]
and on $\mathcal{E}_R^+$,
\[
\sup_{k\ge0} \ \norm{{\btheta}_{m^*+k}}
\le
\rho
\max\left\{
\norm{{\btheta}_{m^*}-{\btheta}^*},
\norm{{\btheta}^*},
\frac{\rinf}{\phiinf}
\right\}
\le
\rho\rbase
=
\rmax\,.
\]
The constants $A_1(\delta)$, $A_2$, and $c_{\min}(\delta)$ are exactly those obtained from the last square-summability bound and the confidence parameter $\delta/4$.
Taking expectations in the preceding conditional probability bound gives $\Pr\{\mathcal{E}_R^+\}\ge1-\delta/4$.
On the finite prefix, Lemma~\ref{lem:deterministic-prefix} gives
$\sup_{0\le s\le m^*} \ \norm{{\btheta}_s}\le B_{m^*}\le\rbase\le\rmax$.
Therefore the event
\[
\mathcal E_R
\coloneqq
\left\{\sup_{t\ge0} \ \norm{{\btheta}_t}\le\rmax\right\}
\]
satisfies
\[
\Pr\{\mathcal E_R\}\ge1-\frac{\delta}{4}\,.
\]
We next follow the proof strategy of Theorem~\ref{thm:PR-main} in Section~\ref{app:proof-PR-main}, spelling out the high-probability events needed for the present mixing-free stepsize.
For every $T\ge1$, the energy decomposition~\eqref{eq:main_decom}, which will be used for the robust-rate bound, gives
\begin{align}
\label{eq:cor20-decom-robust}
\|\be_T\|^2
&=\|\be_0\|^2
+ \sum_{t=1}^T \eta^2_{t-1} \|\bg({\btheta}_{t-1},Z_{t-1})\|^2
+ 2\sum_{t=1}^T \eta_{t-1}\langle\bar \bg({\btheta}_{t-1}),{\btheta}_{t-1}-{\btheta}^*\rangle
+ B_T,
\end{align}
where $\be_t={\btheta}_t-{\btheta}^*$ and
\[
B_T
=
2\sum_{t=1}^T \eta_{t-1}\,
\langle \bg({\btheta}_{t-1},Z_{t-1})-\bar \bg({\btheta}_{t-1}),\,{\btheta}_{t-1}-{\btheta}^*\rangle\,.
\]
Similarly, the PR decomposition~\eqref{eq:PR-decomp-eta}, which will be used for the fast-rate bound, gives
\begin{align}
\label{eq:cor20-decom-fast}
\bA\bar{\be}_T=I_{1,T}+I_{2,T}+I_{3,T},
\end{align}
where
\[
\bar{\be}_T
:=
\frac{1}{S_T}\sum_{t=1}^{T}\eta_{t-1} \be_{t-1},
\]
and
\[
I_{1,T}=\frac{1}{S_T}\sum_{t=1}^T(\be_{t-1}-\be_t),
\qquad
I_{2,T}=\frac{1}{S_T}\sum_{t=1}^T\eta_{t-1}\bxi_{t-1},
\qquad
I_{3,T}=\frac{1}{S_T}\sum_{t=1}^T\eta_{t-1}\bdelta_{t-1}\be_{t-1}\,.
\]
We now introduce three additional high-probability events, $\mathcal E_2$, $\mathcal E_3$, and $\mathcal E_4$, to control the terms appearing in the two decompositions above.
First, applying Lemma~\ref{lem:I2-bound} with confidence parameter $\delta/4$ and using
$\rinf\phiinf+2\phiinfs\norm{{\btheta}^*}\le3\phiinfs\rmax$, we obtain an event $\mathcal E_2$ such that
\[
\Pr\{\mathcal E_2\}\ge1-\frac{\delta}{4},
\]
and, on $\mathcal E_2$, for all $T\ge1$,
\[
\norm{I_{2,T}}
\le \frac{48\mixing\phiinfs\rmax}{S_T} \left(2\sqrt{H}L_\delta+3\eta_0\right)\,.
\]
Next, to control $I_{3,T}$ in~\eqref{eq:cor20-decom-fast} and the Markov-bias term $B_T$ in~\eqref{eq:cor20-decom-robust}, we use the stopped process at radius $\rmax$:
\[
\exit\coloneqq\inf\{t\ge0:\norm{{\btheta}_t}>\rmax\},
\quad
\tilde{\btheta}_t\coloneqq {\btheta}_{t\wedge\exit},
\quad
\tilde \be_t:=\tilde{\btheta}_t-{\btheta}^*,
\quad
\tilde\eta_t\coloneqq \eta_t\indicator\{t<\exit\}\,.
\]
Define the corresponding localized quantities:
\[
\tilde I_{3,T}
:=
\frac{1}{S_T}\sum_{t=1}^{T}\tilde\eta_{t-1}\bdelta_{t-1} \tilde \be_{t-1},
\qquad
\tilde B_T
:=
2\sum_{t=1}^T \tilde\eta_{t-1}\,
\langle \bg(\tilde{\btheta}_{t-1},Z_{t-1})-\bar \bg(\tilde{\btheta}_{t-1}),\,\tilde{\btheta}_{t-1}-{\btheta}^*\rangle\,.
\]
Applying Lemma~\ref{lem:I3-bound} with confidence parameter $\delta/4$ gives an event $\mathcal E_3$ such that
\[
\Pr\{\mathcal E_3\}\ge1-\frac{\delta}{4},
\]
and, on $\mathcal E_3$, for all $T\ge1$,
\[
\norm{\tilde I_{3,T}}
\le \frac{1}{S_T} \left(
256\mixing\phiinfs\rmax\sqrt{H}L_\delta
+384\eta_0\mixing\phiinfs\rmax
+192\mixing\phi_\infty^4\rmax H
\right)\,.
\]
Finally, Lemma~\ref{lem:Btilde-hp-any} with confidence parameter $\delta/4$, together with the deterministic quadratic-term bound used in the proof of Lemma~\ref{lem:bootstrap-ineq}, gives an event $\mathcal E_4$ such that
\[
\Pr\{\mathcal E_4\}\ge1-\frac{\delta}{4},
\]
and, on $\mathcal E_4$, for all $T\ge1$,
\begin{equation}
\label{eq:cor20-energy-control}
\begin{aligned}
\sum_{t=1}^T\tilde\eta_{t-1}^2 \norm{\bg(\tilde{\btheta}_{t-1},Z_{t-1})}^2 + \tilde B_T
&\le 768\mixing\phiinfs\rmaxs\sqrt{H}L_\delta +1152\eta_0\mixing\phiinfs\rmaxs\\
&\quad +1344\mixing\phi_\infty^4\rmaxs H +9\phi_\infty^4\rmaxs H\,.
\end{aligned}
\end{equation}
Define
\[
\mathcal E
\coloneqq \mathcal E_R\cap\mathcal E_2\cap\mathcal E_3\cap\mathcal E_4\,.
\]
Since each of $\mathcal E_R$, $\mathcal E_2$, $\mathcal E_3$, and $\mathcal E_4$ has probability at least $1-\delta/4$, the union bound gives $\Pr\{\mathcal E\}\ge1-\delta$.
On $\mathcal E_R$, we have $\exit=\infty$, so the stopped and original processes coincide. Hence, on $\mathcal E$, $\tilde I_{3,T}=I_{3,T}$ and $\tilde B_T=B_T$ for every $T\ge1$.

We now derive the fast-rate control on $\mathcal E$.  By Lemma~\ref{lem:I1-bound},
\[
\norm{I_{1,T}}\le\frac{2\rmax}{S_T}\,.
\]
Combining~\eqref{eq:cor20-decom-fast} with the preceding bounds on $I_{1,T}$, $I_{2,T}$, and $\tilde I_{3,T}=I_{3,T}$, for all $T\ge1$,
\begin{equation}
\label{eq:cor20-Ae-fast-preconstant}
\norm{\bA\bar{\be}_T}
\le
\frac{\rmax}{S_T}
\left[
2
+528\eta_0\mixing\phiinfs
+352\mixing\phiinfs\sqrt{H}L_\delta
+192\mixing\phi_\infty^4 H
\right]\,.
\end{equation}
For the present stepsize,
\begin{equation}
\label{eq:cor20-eta0-H-bounds}
\eta_0=\frac{1}{c\phiinfs\log^2 3},
\qquad
H
= \frac{1}{c^2\phi_\infty^4}
\sum_{t=0}^\infty\frac{1}{(t+1)\log^4(t+3)}
\le \frac{3}{c^2\phi_\infty^4\log 2},
\end{equation}
where the last inequality follows from $\log(t+3)\ge\log 3>1$ and Lemma~\ref{lem:stepsize-lower-bound}(i).
Hence
\[
\begin{aligned}
&2
+528\eta_0\mixing\phiinfs
+352\mixing\phiinfs\sqrt{H}L_\delta
+192\mixing\phi_\infty^4 H
\\
&\qquad\le 2\left( 1+ \frac{264\mixing}{c\log^2 3} + \frac{176\sqrt{3}\mixing}{c\sqrt{\log 2}}L_\delta + \frac{288\mixing}{c^2\log 2} \right)\,.
\end{aligned}
\]
Using~\eqref{eq:fast_rate_decomp}, \eqref{eq:cor20-Ae-fast-preconstant}, and the preceding simplification,
\begin{equation}
\label{eq:cor20-fast-before-ST}
\begin{aligned}
f(\bar{\btheta}_T)-f({\btheta}^*)
&\le
\frac{4\rmaxs}{\omega S_T^2}
\left(
1+
\frac{264\mixing}{c\log^2 3}
+
\frac{176\sqrt{3}\mixing}{c\sqrt{\log 2}}L_\delta
+
\frac{288\mixing}{c^2\log 2}
\right)^2\,.
\end{aligned}
\end{equation}

We next derive the robust-rate control on $\mathcal E$.
Using the robust rate decomposition~\eqref{eq:cor20-decom-robust}, the identity $\tilde B_T=B_T$ on $\mathcal E$, and~\eqref{eq:cor20-energy-control},
\[
\begin{aligned}
2\sum_{t=1}^T\eta_{t-1}
\Ip{\bar\bg({\btheta}_{t-1})}{{\btheta}^*-{\btheta}_{t-1}}
&\le
\rmaxs
+768\mixing\phiinfs\rmaxs\sqrt{H}L_\delta
+1152\eta_0\mixing\phiinfs\rmaxs
\\
&\quad
+1344\mixing\phi_\infty^4\rmaxs H
+9\phi_\infty^4\rmaxs H \,.
\end{aligned}
\]
By convexity of $f$ and~\eqref{eq:dirichlet-monotone},
\[
\begin{aligned}
f(\bar{\btheta}_T)-f({\btheta}^*)
&\le
\frac{\rmaxs}{S_T}
\left[
\frac12
+384\mixing\phiinfs\sqrt{H}L_\delta
+576\eta_0\mixing\phiinfs
\right.
\\
&\qquad\qquad\left.
+672\mixing\phi_\infty^4H
+\frac92\phi_\infty^4H
\right]\,.
\end{aligned}
\]
Using~\eqref{eq:cor20-eta0-H-bounds},
\[
\begin{aligned}
&\frac12
+384\mixing\phiinfs\sqrt{H}L_\delta
+576\eta_0\mixing\phiinfs
+672\mixing\phi_\infty^4H
+\frac92\phi_\infty^4H
\\
&\qquad\le
\frac12\left(
1+
\frac{1152\mixing}{c\log^2 3}
+
\frac{768\sqrt{3}\mixing}{c\sqrt{\log 2}}L_\delta
+
\frac{4032\mixing+27}{c^2\log 2}
\right)\,.
\end{aligned}
\]
Therefore,
\begin{equation}
\label{eq:cor20-robust-before-ST}
f(\bar{\btheta}_T)-f({\btheta}^*)
\le
\frac{\rmaxs}{2S_T}
\left(
1+
\frac{1152\mixing}{c\log^2 3}
+
\frac{768\sqrt{3}\mixing}{c\sqrt{\log 2}}L_\delta
+
\frac{4032\mixing+27}{c^2\log 2}
\right)\,.
\end{equation}
Finally, Lemma~\ref{lem:stepsize-lower-bound} with $p=2$ and $\kappa=c\phiinfs$ gives
\[
S_T
\ge \frac{2(\sqrt{T+1}-1)}{c\phiinfs\log^2(T+2)},
\]
and hence
\begin{equation}
\label{eq:cor20-ST-inverse-bounds}
\frac{1}{S_T}
\le \frac{c\phiinfs\log^2(T+2)}{2(\sqrt{T+1}-1)},
\qquad
\frac{1}{S_T^2}
\le \frac{c^2\phi_\infty^4\log^4(T+2)}{4(\sqrt{T+1}-1)^2}\,.
\end{equation}
Substituting~\eqref{eq:cor20-ST-inverse-bounds} into~\eqref{eq:cor20-fast-before-ST} and~\eqref{eq:cor20-robust-before-ST} gives the displayed explicit rates.
The big-$\widetilde{\mathcal O}$ statement follows directly from the displayed formula, and the doubly-exponential dependence on $\mixing$ follows from Lemma~\ref{lem:deterministic-prefix} and $m^*\le e^{\mixing}$.
\end{proof}
\section{Technical Lemmas}

\subsection{Pinelis' Inequality}
\begin{lemma}[{\citealt[Thm.~3.5]{pinelis1994optimum}}]
\label{lem:pinelis}
Let $(f_j)_{j\ge 0}$ be a martingale in a $(2,D)$-smooth Banach space $(X,\|\cdot\|)$ with $f_0=0$.
Let $d_j:=f_j-f_{j-1}$, $f^\ast:=\sup_{j\ge 0} \ \|f_j\|$, and $\|d_j\|_\infty:=\operatorname*{ess\,sup}\|d_j\|$.
If $\sum_{j\ge 1}\|d_j\|_\infty^2 \le b_\ast^2$, then for all $r\ge 0$,
\[
\mathbb{P}\!\left(f^\ast \ge r\right)\le 2\exp\!\left(-\frac{r^2}{2D^2 b_\ast^2}\right)\,.
\]
\end{lemma}

\subsection{Miscellaneous lemmas}

\begin{lemma}[Dirichlet matrix]
\label{lem:L-dir-properties}
The Dirichlet matrix
\[
\bL_{\mathrm{Dir}}
:=
\bD-\frac{1}{2}\bigl(\bD\bP^\mu+(\bP^\mu)^\top\bD\bigr)
\]
is symmetric and positive semidefinite. More precisely, if
\[
\bP_{\mathrm{sym}}
:=
\frac{1}{2}\Bigl(\bP^\mu+\bD^{-1}(\bP^\mu)^\top\bD\Bigr),
\]
then, for every \(\bx\in\mathbb{R}^n\),
\[
\bx^\top \bL_{\mathrm{Dir}} \bx
=
\frac{1}{2}\sum_{i,j}\pi(i)\bP_{\mathrm{sym}}(i,j)(x_i-x_j)^2
=
\frac{1}{2}\sum_{i,j}\pi(i)P^\mu(i,j)(x_i-x_j)^2
\ge 0\,.
\]
\end{lemma}

\begin{proof}
First, \(\bL_{\mathrm{Dir}}\) is symmetric because \(\bD^\top=\bD\):
\[
\bL_{\mathrm{Dir}}^\top
=
\left(\bD-\frac{1}{2}\bigl(\bD\bP^\mu+(\bP^\mu)^\top\bD\bigr)\right)^\top
=
\bD-\frac{1}{2}\bigl((\bP^\mu)^\top\bD+\bD\bP^\mu\bigr)
=
\bL_{\mathrm{Dir}}\,.
\]

Define the time-reversal kernel
\[
\bP_{\mathrm{rev}}^\mu
:=
\bD^{-1}(\bP^\mu)^\top\bD,
\qquad
\bP_{\mathrm{rev}}^\mu(i,j)
=
\frac{\pi(j)P^\mu(j,i)}{\pi(i)}\,.
\]
Since \(\pi^\top\bP^\mu=\pi^\top\), \(\bP_{\mathrm{rev}}^\mu\) is row-stochastic:
\[
\sum_j \bP_{\mathrm{rev}}^\mu(i,j)
=
\frac{1}{\pi(i)}\sum_j\pi(j)P^\mu(j,i)
=
1\,.
\]
Hence \(\bP_{\mathrm{sym}}=\frac12(\bP^\mu+\bP_{\mathrm{rev}}^\mu)\) is also row-stochastic and satisfies detailed balance:
\[
\pi(i)\bP_{\mathrm{sym}}(i,j)
=
\frac12\bigl(\pi(i)P^\mu(i,j)+\pi(j)P^\mu(j,i)\bigr)
=
\pi(j)\bP_{\mathrm{sym}}(j,i)\,.
\]
Moreover,
\[
\bD\bP_{\mathrm{sym}}
=
\frac12\bigl(\bD\bP^\mu+(\bP^\mu)^\top\bD\bigr),
\]
so \(\bL_{\mathrm{Dir}}=\bD-\bD\bP_{\mathrm{sym}}\). Therefore, for any \(\bx\in\mathbb{R}^n\),
\[
\begin{aligned}
\bx^\top\bL_{\mathrm{Dir}}\bx
&=\sum_i\pi(i)x_i^2-\sum_{i,j}\pi(i)\bP_{\mathrm{sym}}(i,j)x_i x_j \\
&=\frac12\sum_{i,j}\pi(i)\bP_{\mathrm{sym}}(i,j)(x_i-x_j)^2
\ge0,
\end{aligned}
\]
where the second equality uses that \(\bP_{\mathrm{sym}}\) is row-stochastic and satisfies detailed balance.
Finally, using the displayed identity for \(\pi(i)\bP_{\mathrm{sym}}(i,j)\) and swapping indices in the second term gives
\[
\frac12\sum_{i,j}\pi(i)\bP_{\mathrm{sym}}(i,j)(x_i-x_j)^2
=
\frac12\sum_{i,j}\pi(i)P^\mu(i,j)(x_i-x_j)^2\,.
\]
This proves the claimed identity and \(\bL_{\mathrm{Dir}}\succeq0\).
\end{proof}

\begin{lemma}[Logarithmic stepsize estimates]
\label{lem:stepsize-lower-bound}
\label{lem:stepsize-square-summable}
The following estimates hold, where $\log(\cdot)$ denotes the natural logarithm.

\emph{(i) Square summability.} Let $(a_t)_{t\ge 0}$ be defined by
\[
a_t
= \frac{1}{\sqrt{t+1}\,\log(t+3)} \qquad \forall t \ge 0\,.
\]
Define the tail sums
\[
Q_T \;:=\; \sum_{t=T}^{\infty} a_t^2, \qquad T\in\mathbb{N}\cup\{0\}\,.
\]
Then, for all integers $T\ge 0$,
\begin{equation}\label{eq:eta-square-tail}
Q_T
\;\le\; \frac{3}{\log(T+2)}\,.
\end{equation}
In particular,
\[
\sum_{t=0}^{\infty}a_t^2
\;\le\;
\frac{3}{\log 2}
\;<\;\infty\,.
\]

\emph{(ii) Partial-sum lower bound.} Let $p\ge 0$, $\kappa>0$, and let $(\eta_t)_{t\ge0}$ be defined by
\[
\eta_t
= \frac{1}{\kappa\sqrt{t+1}\,\log^p(t+3)}, \qquad t\ge0\,.
\]
For every integer $T\ge1$, define
\[
S_T\coloneqq \sum_{t=1}^{T}\eta_{t-1}\,.
\]
Then
\[
S_T
\ge \frac{2(\sqrt{T+1}-1)}{\kappa\log^p(T+2)}\,.
\]
\end{lemma}

\begin{proof}
For (i), by definition,
\[
a_t^2
= \frac{1}{(t+1)\,\log^2(t+3)}\,.
\]
For every $t\ge 0$, we have $t+1 \ge \frac{t+3}{3}$, hence
\[
\frac{1}{(t+1)\log^2(t+3)}
\le \frac{3}{(t+3)\log^2(t+3)}\,.
\]
Therefore,
\[
Q_T
\le \sum_{t=T}^{\infty}\frac{3}{(t+3)\log^2(t+3)}\,.
\]
Let $n=t+3$. Then
\[
\sum_{t=T}^{\infty}\frac{1}{(t+3)\log^2(t+3)}
= \sum_{n=T+3}^{\infty}\frac{1}{n\log^2 n}\,.
\]
Now define $f(x)=\frac{1}{x\log^2 x}$ for $x>1$. Since $f$ is decreasing on $(1,\infty)$,
the integral test yields, for all $N\ge 3$,
\[
\sum_{n=N}^{\infty} f(n)
\le
\int_{N-1}^{\infty} f(x)\,dx
=
\int_{N-1}^{\infty}\frac{1}{x\log^2 x}\,dx
=
\frac{1}{\log(N-1)}\,.
\]
Applying this with $N=T+3$ gives
\[
\sum_{n=T+3}^{\infty}\frac{1}{n\log^2 n}
\le \frac{1}{\log(T+2)}\,.
\]
Substituting back proves \eqref{eq:eta-square-tail}. Taking $T=0$ gives the claimed finite upper bound on
$\sum_{t=0}^{\infty}a_t^2$.

For (ii), let $s=t$. Then
\[
S_T
= \frac{1}{\kappa}\sum_{s=1}^{T}\frac{1}{\sqrt{s}\,\log^p(s+2)}\,.
\]
Since $p\ge0$, $\log(\cdot)$ is increasing, and $s+2\le T+2$ for all $1\le s\le T$,
\[
\frac{1}{\sqrt{s}\,\log^p(s+2)}
\ge \frac{1}{\sqrt{s}\,\log^p(T+2)}\,.
\]
Therefore,
\[
S_T
\ge \frac{1}{\kappa\log^p(T+2)} \sum_{s=1}^{T}\frac{1}{\sqrt{s}}\,.
\]
The function $x\mapsto x^{-1/2}$ is decreasing on $[1,\infty)$, hence
\[
\sum_{s=1}^{T}\frac{1}{\sqrt{s}}
\ge \int_{1}^{T+1}\frac{1}{\sqrt{x}}\,dx
= 2(\sqrt{T+1}-1)\,.
\]
Combining the last two displays proves the claim.
\end{proof}

\end{document}